\documentclass[pdflatex,sn-basic,iicol]{sn-jnl}

\usepackage{graphicx}%
\usepackage{multirow}%
\usepackage{amsmath,amssymb,amsfonts}%
\usepackage{amsthm}%
\usepackage{mathrsfs}%
\usepackage[title]{appendix}%
\usepackage{textcomp}%
\usepackage{manyfoot}%
\usepackage{booktabs}%
\usepackage{algpseudocode}%
\usepackage{listings}%
\usepackage{subcaption}
\usepackage{lscape}
\usepackage[ruled, lined, linesnumbered, commentsnumbered, longend]{algorithm2e}
\usepackage{enumitem}
\usepackage{adjustbox}
\usepackage{makecell}
\usepackage{stfloats}
\usepackage{array}
\usepackage{pifont}
\newcommand{\cmark}{\ding{51}}%
\newcommand{\xmark}{\ding{55}}%
\usepackage[table]{xcolor}
\usepackage{balance}

\theoremstyle{thmstyleone}%
\newtheorem{lemma}{Lemma}

\theoremstyle{thmstyletwo}%

\theoremstyle{thmstylethree}%
\newtheorem{definition}{Definition}%

\begin{document}

\title[Near OOD Detection for Vision-Language Prompt Learning with Contrastive Logit Score]{Near OOD Detection for Vision-Language Prompt Learning with Contrastive Logit Score}


\author[1]{\fnm{Myong Chol} \sur{Jung}}\email{davidmcjung@gmail.com}

\author[3]{\fnm{Joanna} \sur{Dipnall}}\email{jo.dipnall@monash.edu}
\author[3]{\fnm{Belinda} \sur{Gabbe}}\email{belinda.gabbe@monash.edu}
\author[2,1]{\fnm{He} \sur{Zhao}}\email{he.zhao@data61.csiro.au}
\equalcont{Corresponding author}


\affil[1]{\orgdiv{Faculty of Information Technology}, \orgname{Monash University}, \orgaddress{\country{Australia}}}

\affil[2]{\orgname{CSIRO's Data61}, \orgaddress{\country{Australia}}}


\affil[3]{\orgdiv{School of Public Health and Preventive Medicine}, \orgname{Monash University}, \orgaddress{\country{Australia}}}


\abstract{Prompt learning has emerged as an efficient and effective method for fine-tuning vision-language models such as CLIP. While many studies have explored generalisation abilities of these models in few-shot classification tasks and a few studies have addressed far out-of-distribution (OOD) of the models, their potential for addressing near OOD detection remains underexplored. 
Existing methods either require training from scratch, need fine-tuning, or are not designed for vision-language prompt learning.
To address this, we introduce the Contrastive Logit Score (CLS), a novel post-hoc, plug-and-play scoring function.
CLS significantly improves near OOD detection of pre-trained vision-language prompt learning methods without modifying their model architectures or requiring retraining. Our method achieves up to an 11.67\% improvement in AUROC for near OOD detection with minimal computational overhead. Extensive evaluations validate the effectiveness, efficiency, and generalisability of our approach. Our codes are available at \url{https://anonymous.4open.science/r/near OOD-prompt-learning-25D1}.}




\maketitle

\section{Introduction}
Pre-trained vision-language models such as ALIGN~\citep{jia2021scaling} and CLIP~\citep{radford2021learning} have shown remarkable visual-text understanding by learning to align image features and textual features of a large-scale image-text dataset via contrastive learning. Consequently, CLIP naturally excels at zero-shot image classification, utilising a class name as natural language text instead of arbitrarily numbered category. For instance, cosine similarity between encoded image feature of an image and encoded textual feature of ``a photo of a [CLASS].'' with can be used as the classification logit of the ``[CLASS]''. The resulting logits for different class names are then converted into probabilities using the softmax function, enabling classification without requiring an additional classification head.

While CLIP’s zero-shot capabilities are impressive, recent studies have highlighted its sensitivity to prompt wording. For instance, \citet{zhou2022learning} demonstrated that small variations in prompt structure (e.g., ``a photo of a [CLASS]'' vs. ``a photo of [CLASS]'') can lead to significant decrease in accuracy, which sometimes exceeds 5\% on standard benchmarks like Caltech101~\citep{li2004learning}. This observation has led to an emerging research direction of prompt learning for few-shot classification with vision-language models~\citep{zhou2022learning,zhou2022conditional,yao2023visual,zhu2023prompt,khattak2023maple,khattak2023self}, which optimises continuous context vectors in the word-embedding space to eliminate the need of handcrafting prompts.

Although existing methods have shown success in this area, the majority focus primarily on enhancing classification accuracy, often overlooking the equally critical challenge of out-of-distribution (OOD) detection~\citep{yang2024generalized}. OOD detection is crucial for real-world, especially in safety-critical applications such as autonomous driving, healthcare, and industrial automation, where models must perform reliably under unfamiliar or unexpected conditions. In these applications, excelling in in-distribution (ID) classification alone is not enough. Models must also be capable of detecting and effectively handling OOD samples. In this paper, we particularly focus on the task of \textit{near OOD detection}~\citep{yang2023full,yang2022openood,zhang2023openood,fort2021exploring,ren2021simple,winkens2020contrastive}, where near OOD datasets only have semantic shift (e.g., both ID and near OOD samples are flower images with no overlapping classes), while far OOD datasets further contain obvious covariate or domain shift (e.g., images of dogs as an ID dataset and images of digits as an OOD dataset)~\citep{yang2024generalized}.

\begin{table*}[!t]
    \centering
    \caption{Comparison between related methods.}
    \resizebox{0.9\textwidth}{!}{\centering
    \begin{tabular}{cc>{\centering\arraybackslash}m{5cm}>{\centering\arraybackslash}m{5cm}}
        \toprule
         &  Training-Free & Tailored for CLIP prompt learning & Applicable to any CLIP prompt learning method \\
         \midrule
         Mahalanobis~\citep{lee2018simple} & \cmark & \xmark & \xmark  \\
         Energy~\citep{liu2020energy} & \cmark & \xmark & \cmark \\
         MaxLogit~\citep{hendrycks2022scaling} & \cmark & \xmark  & \cmark  \\
         CLIPN~\citep{wang2023clipn} & \xmark & \xmark & \xmark \\
         LoCoOp~\citep{miyai2023locoop} & \xmark & \cmark & \xmark\\
         NegLabel~\citep{jiang2024negative} & \cmark & \xmark & \xmark\\
         NegPrompt~\citep{Li2024learning} & \xmark & \cmark & \xmark\\
         CLS (Ours) & \cmark & \cmark & \cmark\\
         \bottomrule
    \end{tabular}}
    \label{tab:checklist}
\end{table*}

As an emerging research direction, several recent works have explored OOD detection in the context of vision-language prompt learning~\citep{Li2024learning,miyai2023locoop}. However, most existing methods are single-purpose approaches that require fine-tuning, limiting their flexibility and applicability. In contrast, we focus on scoring functions that can be directly computed from the output logits of any prompt learning method without requiring training or retraining. Our goal is to develop a plug-and-play solution that effectively distinguishes between ID and OOD samples.

Several existing OOD scoring functions have been proposed for the general OOD detection problem~\citep{hendrycks2022scaling,liu2020energy,lee2018simple}. While these methods can be applied in a post-hoc manner, they are not specifically designed for near OOD detection or vision-language models like CLIP. As a result, the overlapping score distributions between ID and OOD samples in these models often lead to poor detection performance. 

To bridge the gap, we propose a novel scoring function, namely Contrastive Logit Score (CLS), specifically designed for vision-language models. Rooted in a hypothesis testing perspective, CLS leverages the contrastive nature of CLIP to effectively distinguish ID and OOD samples. CLS enhances the separation between ID and near OOD samples, leading to substantial performance improvements. Notably, our method requires no modifications to the model architecture and does not involve retraining, making it a highly efficient and adaptable post-hoc solution for vision-language models. The unique advantage of our method can be reflected in Table~\ref{tab:checklist} by comparing with closely related works. 

We validate our method across 13 diverse datasets and 8 state-of-the-art prompt learning models. Our experiments show that our framework improves near OOD detection performance by up to 11.67\% in terms of AUROC, without affecting the classification accuracy of the underlying models. This demonstrates the versatility and effectiveness of our approach in real-world applications.

\section{Background}
\subsection{Contrastive Language–Image Pre-training (CLIP)}
Contrastive Language–Image Pre-training (CLIP)~\citep{radford2021learning} is a widely adopted dual-encoder vision–language model trained on 400 million image–text pairs and is well known for its strong zero-shot classification capabilities. Although we use CLIP as the primary backbone in our study due to its established benchmark status and broad adoption, our framework is not restricted to CLIP. The proposed method is readily applicable to other dual-encoder vision–language models that share a similar embedding structure, including ALIGN~\citep{jia2021scaling}, LiT~\citep{zhai2022lit}, OpenCLIP~\citep{Cherti_2023_CVPR}, EVA-CLIP~\citep{sun2023eva}, and SigLIP~\citep{Zhai_2023_ICCV}. These models all produce aligned visual and textual representations, allowing seamless integration with our approach without modification to the underlying architecture. Given the widespread use of CLIP and its successors across downstream tasks such as zero-shot and few-shot classification~\citep{radford2021learning,zhou2022learning}, object detection~\citep{gu2021open}, image segmentation~\citep{Xu_2022_CVPR}, image generation and editing~\citep{ramesh2022hierarchical,Patashnik_2021_ICCV}, retrieval, captioning~\citep{yu2022coca}, and multimodal reasoning~\citep{pmlr-v202-li23q}, compatibility with this family of models underscores the generality and practical applicability of our framework.

Building on this foundation, we briefly review how CLIP performs classification. CLIP measures cosine similarity between image feature of an unseen image and textual feature of a text prompt formatted as ``a photo of a [CLASS]'' where [CLASS] is the name of a class in a label space of interest. Formally, given an image $I\in \mathbb{R}^{H\times W \times 3}$ with $H$ being the height and $W$ being the width and a text prompt $T=\text{``a photo of a [CLASS]''}$, a classification logit is computed by $\langle\text{Enc}_{\text{I}}(I),\text{Enc}_{\text{T}}(T)\rangle$ where $\langle\cdot,\cdot\rangle$ is cosine similarity, $\text{Enc}_{\text{I}}(\cdot)$ is an image encoder, and $\text{Enc}_{\text{T}}(\cdot)$ is a text encoder. The image encoder can be either ResNet~\citep{he2016deep} or Vision Transformer (ViT)~\citep{dosovitskiy2021an}, and the text encoder is Transformer~\citep{vaswani2017attention}. For the brevity of notation, we omit the notations of the encoders for the remainder of the paper.

\subsection{Prompt Learning of CLIP} Prompt learning of CLIP was first introduced by \citet{zhou2022learning} through Context Optimization (CoOp), which adapts popular prompt learning techniques from the natural language processing (NLP) field to CLIP. It addresses the issue of CLIP's classification being sensitive to the prompt's prefix  (e.g., a large performance gap between when using ``a photo of a [CLASS]'' and ``a [CLASS]'') by optimising the prefix with few-shot samples. CoOp learns $M$ continuous context vectors $V=\{V_1, V_2, \cdots,V_M\}$ within word-embedding space where $V_i \in \mathbb{R}^D$ is the $i^{th}$ vector with $D$ being the word-embedding dimension. The learnable prompt is formalised as $P=\{V_1, V_2, \cdots, V_M, C\}$ where $C\in \mathbb{R}^D$ is the word-embedding of a class name appended to the context vectors. The classification logit of $i^{th}$ class is then computed by $\langle I, P_i\rangle$ where $P_i$ is the learnable prompt with the $i^{th}$ class name. The probability is estimated by the softmax function as $p(y=i\vert I,P_i)=\frac{\exp{(\langle I, P_i\rangle/\tau)}}{\sum_{k=1}^K{\exp{(\langle I,P_k\rangle/\tau)}}}$
where $K$ is the total number of classes and $\tau$ is the temperature scale. Cross-entropy loss is then minimised to learn the context vectors. Note that the only learnable parameters that are common among different prompt learning models are the $M$ context vectors. Since the introduction of CoOp, a number of subsequent works have aimed to improve its ID accuracy and generalisability with modifications in model architecture or additional loss terms (See Section~\ref{sec-rw}). The aim of the paper is to develop a post-hoc approach that improves near OOD detection performance while being agnostic to base prompt learning models.

\subsection{Near Out-Of-Distribution Detection}  

OOD detection is largely categorised as far OOD detection and near OOD detection based on the distribution shift between an ID test dataset and an OOD dataset along with difficulty of detection~\citep{ren2021simple,fort2021exploring,yang2021generalized,yang2022openood,zhang2023openood}. Far OOD datasets have covariate shift in images (i.e., OOD samples are from domains that differ from the training set), and near OOD datasets which are more challenging to detect involve semantic shift (i.e., OOD samples are drawn from the same domain as the training set but belong to previously unseen label classes). Near OOD detection is also synonymous with fine-grained OOD detection~\citep{zhang2023mixture} and hard OOD detection~\citep{Li2024learning,ming2022delving}.

In this paper, we study near OOD detection via prompt learning of CLIP, a new research problem to which no existing methods are tailored. Given a trained CLIP prompt learning method, we focus on post-hoc approaches that compute a score from the logits of the method to determine whether a given image is from ID or from OOD, which can be written as:
\begin{equation}
    g\left(I;\{P_i\}_{i=1}^K,\alpha\right)=\begin{cases}
    1 & S\left(I;\{P_i\}_{i=1}^K\right) \geq \alpha \\ %
    0 & S\left(I;\{P_i\}_{i=1}^K\right)< \alpha
    \end{cases}
\end{equation}
where $g(\cdot)$ is a OOD detector, $\alpha$ is the threshold, and $S(\cdot)$ is a score function. By convention, the ground truth label is 1 for ID samples and 0 for OOD samples. 

\section{Method}

\subsection{Problem Setting}
We focus on a near OOD detection problem for prompt learning models of CLIP, which is to detect whether a given image $I_{\text{test}}$ is from the ID test dataset $\mathcal{D}^{\text{ID}}_{\text{test}}$ of $(I^{\text{ID}}_{\text{test}},y^{\text{ID}})$ pairs or a near OOD dataset $\mathcal{D}^{\text{nearOOD}}_{\text{test}}$ of $(I^{\text{nearOOD}}_{\text{test}},y^{\text{nearOOD}})$ pairs where $y^{\text{ID}}\in\{1,\cdots,K\}$ is the ID label with $K$ classes and $y^{\text{nearOOD}}\in\{1,\cdots,L\}$ is the near OOD label with $L$ classes. The ID dataset and the near OOD dataset contain the same types of images (i.e., no covariate shift in images) but have no overlapping classes (i.e., $y^{\text{ID}}\cap y^{\text{nearOOD}}=\varnothing$). Without loss of generality, we assume that the context vectors $V$ have already been fine-tuned using a prompt learning model with a few-shot ID training dataset and only consider post-training stage in a post-hoc manner.

\subsection{Theoretical Motivations}
We first provide the theoretical motivations of our method  by formulating the OOD detection problem as a hypothesis test with two competing hypotheses. The null hypothesis  $H_{\text{ID}}$ is that the data $x$ is drawn from the in-domain distribution modelled by density $p_{\text{ID}}(x)$ while the alternative hypothesis $H_{\text{OOD}}$ is that the data $x$ is drawn from the near OOD distribution modelled by density $p_{\text{OOD}}(x)$. Our goal is to design a test to classify $x$ as an ID or OOD sample. A commonly used approach is the log likelihood ratio test~\citep{buse1982likelihood}, which is:
\begin{align}
    \lambda(x) = \log \frac{p_{\text{ID}}(x)}{p_{\text{OOD}}(x)} \label{eq-test}
\end{align}
where $x$ is considered as an ID sample when $\lambda(x)$ is larger than a threshold $\gamma$. According to the Neyman–Pearson Lemma~\citep{neyman1933ix}, the likelihood ratio test is the most powerful test for a given Type I error rate, defined as $p\left(\lambda(x) > \gamma \hspace{0.3em} \vert \hspace{0.3em} H_{\text{OOD}}\right)$ (i.e., classifying $x$ as ID when it is actually OOD).

In practice, $p_{\text{ID}}(x)$ and $p_{\text{OOD}}(x)$ are not explicitly known. Inspired by classification with rejection~\citep{chow1970optimum,cortes2016learning,ni2019calibration,charoenphakdee2021classification,hendrickx2024machine} and open set recognition~\citep{rudd2017extreme,geng2020recent,Palechor2023large}, we introduce an additional ``null'' class denoted as $y=0$ to the original multi-class classification setting to take account for samples that do not belong to any of the known in-domain classes. The new ID classes can be rewritten as $y^{\text{ID}*}\in \{0, 1,\dots, K\}$.

\begin{definition}
(Chow’s rule~\citep{chow1970optimum}) Denote $i^* =\arg\max_{i \in \{1, \dots, L\}} p(y=i | x)$ and
the optimal solution of multiclass classification with rejection can be expressed as:
\begin{align}
    y = \begin{cases} 
  0, & p(y=i^*|x) \le 1 - c \\
  i^*, & \text{otherwise}.
\end{cases}
\end{align}
where $c$ is a confidence threshold that controls the trade-off between classification accuracy and rejection rate.
\end{definition}

Chow’s rule suggests that the confidence defined as the maximum probability over the classes can be used as a score function to reject the sample $x$ as an OOD sample:
\begin{align}
p_{\text{ID}}(x) \propto  \max_{i \in \{1, \dots, K\}} p(y=i | x). \label{eq-max}
\end{align}

From the hypothesis-testing perspective, distinguishing ID from OOD requires comparing the likelihood under the in-distribution model to that under an alternative distribution. Since $p_{\text{OOD}}$ is inherently unknown in practice, the classical likelihood-ratio test becomes intractable without an additional modeling assumption. It is common in open set recognition~\citep{rudd2017extreme,geng2020recent} to introduce a ``null'' or ``unknown'' class providing a principled surrogate for the unknown OOD likelihood:
\begin{align}
    p_{\text{OOD}}(x) \approx p(y=0| x).
\end{align}


\subsection{Proposed Score Function}

Intuitively inspired by the log likelihood ratio test, we propose a new score function tailored for CLIP-based prompt learning methods. We emphasise that the hypothesis-testing framing serves as conceptual motivation rather than an exact probabilistic model and our proposed method is designed as a post-hoc score leveraging CLIP’s contrastive structure.

Recall that the predictive probability of CLIP prompt learning by design is:
\begin{align}
p(y=i\vert I,P_i) \propto \exp{(\langle I, P_i\rangle)},    
\end{align}
where the temperature $\tau$ is omitted to assist clarity, $I$ denotes the image representation, and $P_i$ is the prompt representation including the class name.
By Eq.~\eqref{eq-max}, we then have:
\begin{align}
p_{\text{ID}}(x) \propto \max_{i=1}^K \exp{(\langle I, P_i\rangle)},    
\end{align}

Importantly for OOD, we propose:
\begin{align}
p_{\text{OOD}}(x) \approx p(y=0| x) \propto \exp{(\langle I, V\rangle)},
\end{align}
where $\langle I, V\rangle$ represents the cosine similarity between the image feature and the textual feature derived from the fine-tuned context vectors $V$ \emph{without} any class names. Intuitively, the fine-tuned $V$ encodes general ID characteristics without associating them with specific classes, and $\langle I,V\rangle$ quantifies how well an image aligns with these generic, non-class-specific ID features. In this sense, it serves as an indicator of the image’s closeness to the ID distribution. Figure~\ref{fig:dist context} illustrates this relationship: near OOD samples typically exhibit higher $\langle I,V\rangle$ values, whereas far OOD samples yield lower values.
By introducing the null-class vector $V$ and combining it with the max-logit term from the known classes, we obtain a simple yet effective way to perform post-hoc OOD detection within a vision-language prompt learning framework.

Motivated by the log likelihood ratio test in Eq.~\eqref{eq-test}, we derive the following score function:
\begin{align}
    \lambda(x) &= \log \frac{p_{\text{ID}}(x)}{p_{\text{OOD}}(x)} \\
    &= \log p_{\text{ID}}(x) - \log p_{\text{OOD}}(x) \\& =
    \max_{i=1}^K \langle I, P_i\rangle - \langle I, V\rangle.
\end{align}

Although the likelihood-ratio test compares $\log p_{\text{ID}}(x)$ with $\log p_{\text{OOD}}(x)$, in practice both terms must be approximated using surrogate quantities derived from different sources. These surrogates generally do not share a common calibration: their magnitudes, variances, and dynamic ranges may differ, and the surrogate OOD likelihood often absorbs components of epistemic uncertainty rather than pure distributional uncertainty~\citep{Gustafsson2020evaluating,senge2014reliable,malinin2018predictive,nandy2020towards}. As a result, the raw difference between these surrogate scores does not reliably approximate a true log-likelihood ratio. To obtain a meaningful decision statistic, it is therefore necessary to adjust the relative scale of the ID and null evidence, ensuring that their contributions to the test reflect their intended probabilistic roles and enabling more faithful separation between the two hypotheses. 

To address this, 
%
%
%
we introduce a scaling factor $\beta$, that de-correlates the ID and null evidence, ensuring that the resulting discriminant direction best separates ID and OOD samples. With $\beta$,
%
%
%
we derive a new score function called \emph{Contrastive Logit Score (CLS)} defined as:
\begin{align}
    \text{CLS}(x) 
   = \max_{i=1}^K \langle I, P_i\rangle - \beta \langle I, V\rangle.
   \label{eq:CLS}
\end{align}

When $\beta=0$, the proposed CLS score simplifies to the MaxLogit score~\citep{hendrycks2022scaling}. Figure~\ref{fig:explanation diagram} compares the CLS and MaxLogit scores and highlights the contribution of $\langle I,V\rangle$ to improved near OOD detection. While the MaxLogit score effectively separates typical ID samples from far OOD samples based on prediction confidence, it struggles to distinguish between atypical ID samples (i.e., hard-to-classify or rare instances) and near OOD samples. The term $\langle I,V\rangle$, which reflects the presence of generic ID features, provides complementary information to interpret prediction uncertainty. Atypical ID samples often lack distinctive ID features, whereas near OOD samples may exhibit strong ID feature presence. By incorporating $\langle I,V\rangle$ into the CLS formulation (i.e., subtracting the MaxLogit score from it), the two factors become more separable, enhancing discrimination between atypical ID and near OOD samples.

\begin{figure}[!t]
    \centering
    \includegraphics[width=0.45\textwidth]{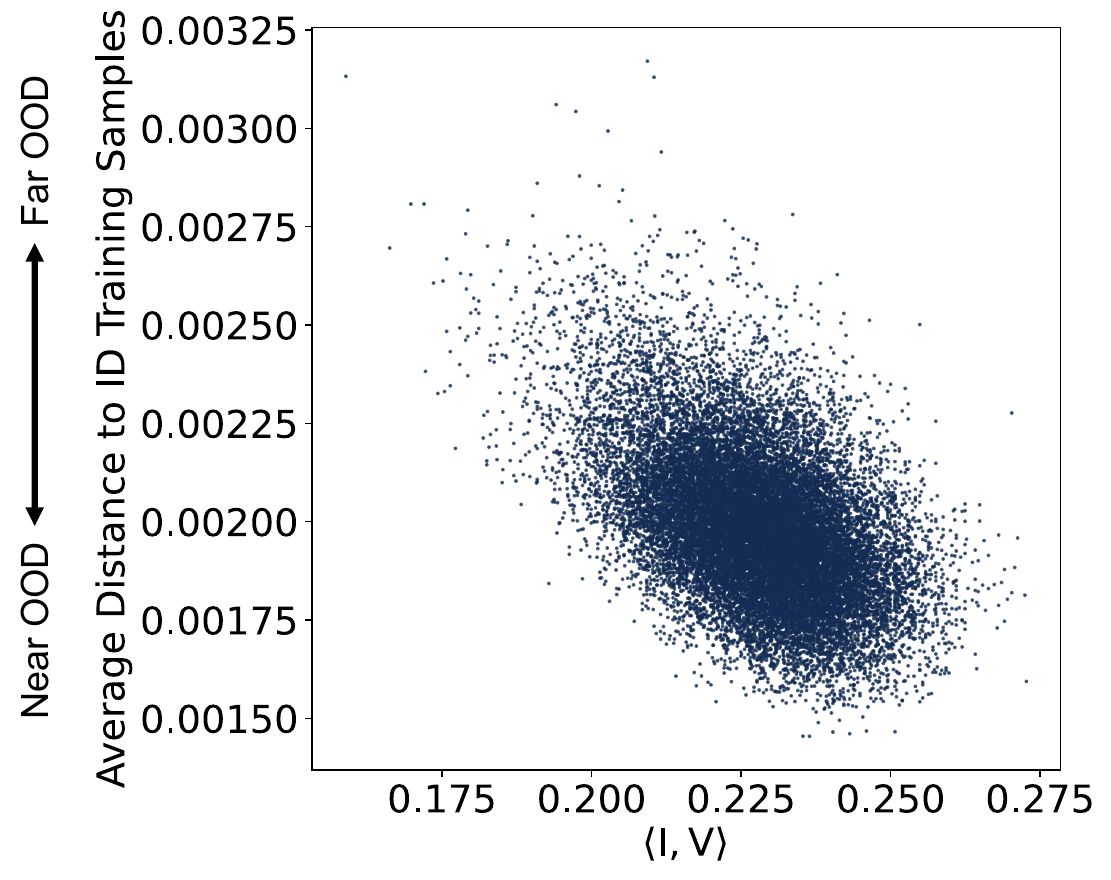}
    \caption{Average Euclidean distance of OOD samples to the ID training samples of ImageNet~\citep{deng2009imagenet} measured with image embeddings of PromptSRC~\citep{khattak2023self} trained with 16 shots.}
    \label{fig:dist context}
\end{figure}
\begin{figure*}[!t]
    \centering
    \includegraphics[width=\textwidth]{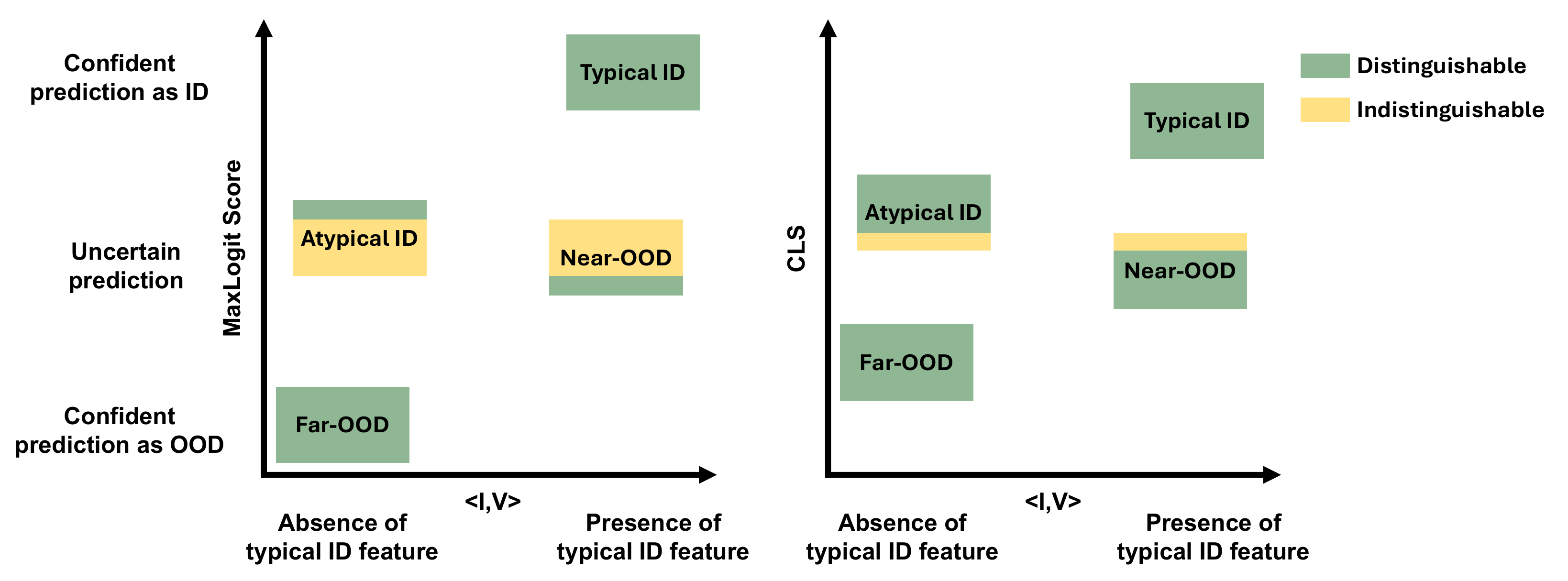}
    \caption{Comparison of CLS and MaxLogit scores. While MaxLogit separates typical ID and far OOD samples, it struggles with atypical ID and near OOD samples. Incorporating $\langle I,V\rangle$, which captures generic ID features, improves separability and enhances near OOD detection.}
    \label{fig:explanation diagram}
\end{figure*}
Figure~\ref{fig:calibration} further demonstrates the effectiveness of CLS compared to the MaxLogit score, where CLS values are plotted on the y-axis and $\langle I,V\rangle$ on the x-axis. Geometrically, subtracting $\langle I,V\rangle$ from the MaxLogit score corresponds to a vertical shearing transformation~\citep{lax2007linear} applied to the score distribution, effectively reducing the overlapping region (highlighted by shaded areas). We observe a consistent performance improvement, measured by AUROC, when $\beta \neq 0$, confirming the advantage of CLS over the MaxLogit score in detecting near OOD samples.

\begin{figure*}[!t]
    \centering
    \begin{subfigure}[c]{0.3\textwidth}
        \centering
        \includegraphics[height=4.3cm]{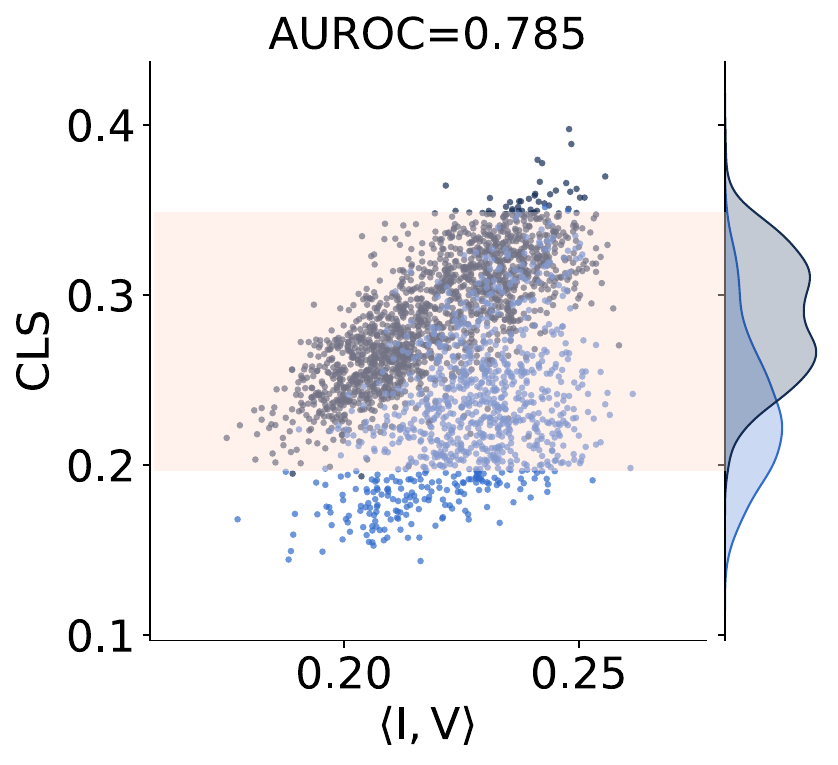}  
        \caption{$\beta=0$}
        \label{fig:maxlogit 2d}
    \end{subfigure}
    \hfill
    \begin{subfigure}[c]{0.3\textwidth}
        \centering
        \includegraphics[height=4.3cm]{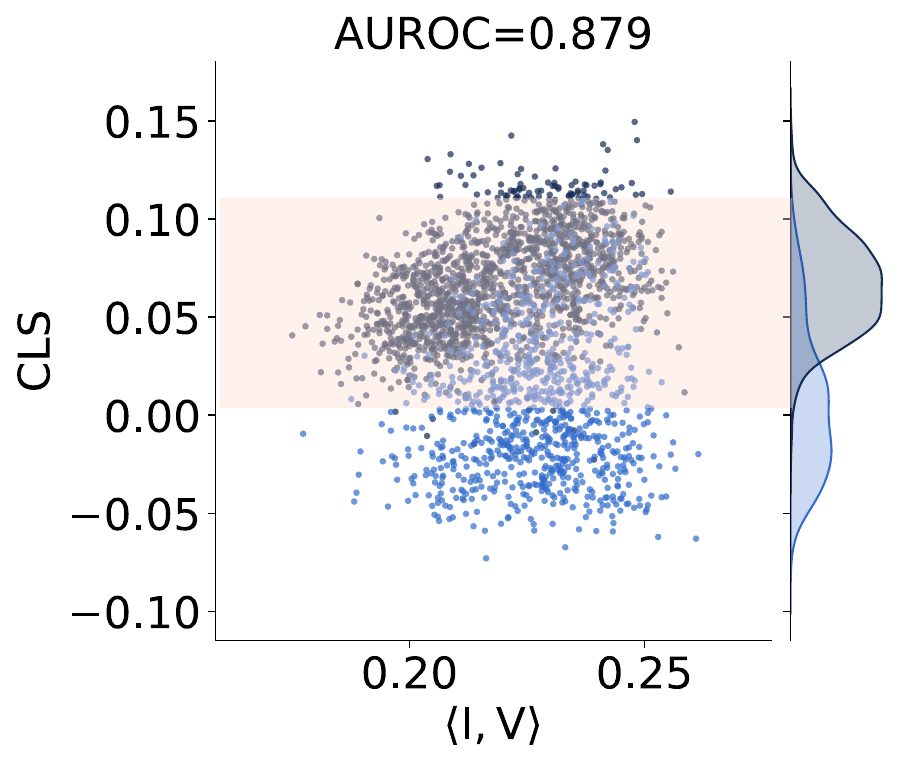}  
        \caption{$\beta=1$}
        \label{fig:CLS-M without beta}
    \end{subfigure}
    \hfill
    \begin{subfigure}[c]{0.3\textwidth}
        \centering
        \includegraphics[height=4.3cm]{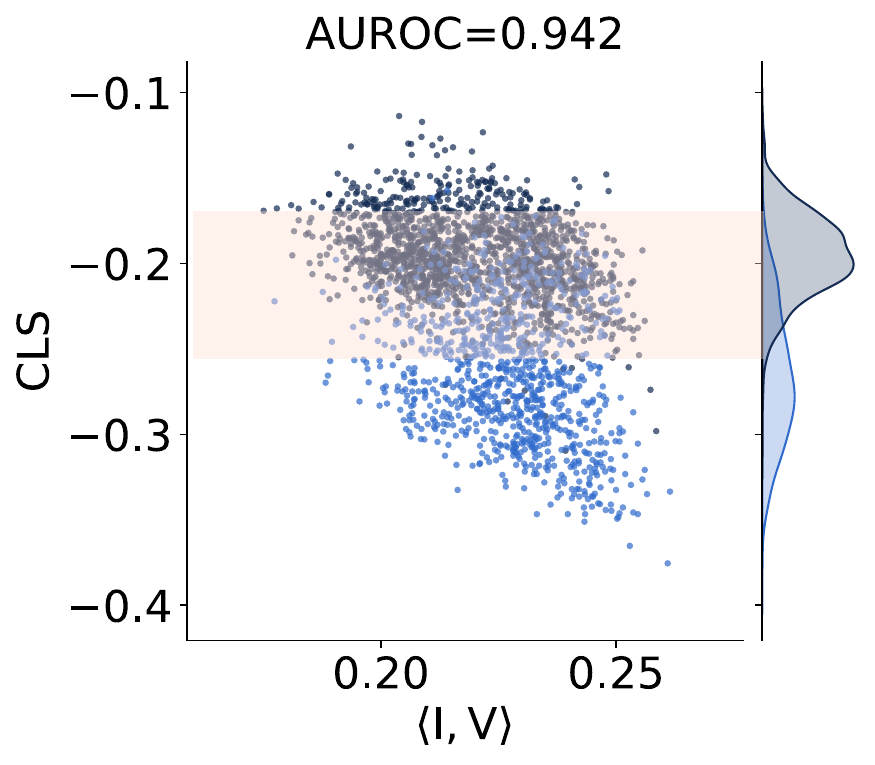}  
        \caption{$\beta=2.2$}
        \label{fig:CLS-M with beta}
    \end{subfigure}
    \begin{subfigure}[c]{0.08\textwidth}
        \centering
        \includegraphics[height=3.1cm]{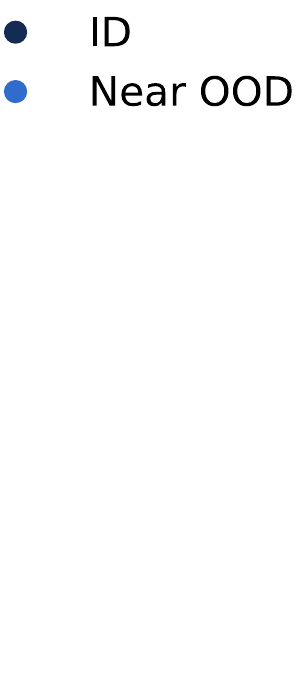}  
    \end{subfigure}
    \caption{(a) MaxLogit score (i.e., CLS with $\beta=0$), (b) CLS with $\beta=1$, and (c) CLS with $\beta=2.2$ of test ID and near OOD samples with respect to Context scores. Areas where ID samples and near OOD samples overlap are highlighted with shaded boxes. All scores are computed using MaPLe~\citep{khattak2023maple} on Caltech101~\citep{li2004learning}.}
    \label{fig:calibration}
\end{figure*}
\begin{figure*}[!t]
    \centering
    \includegraphics[width=\textwidth]{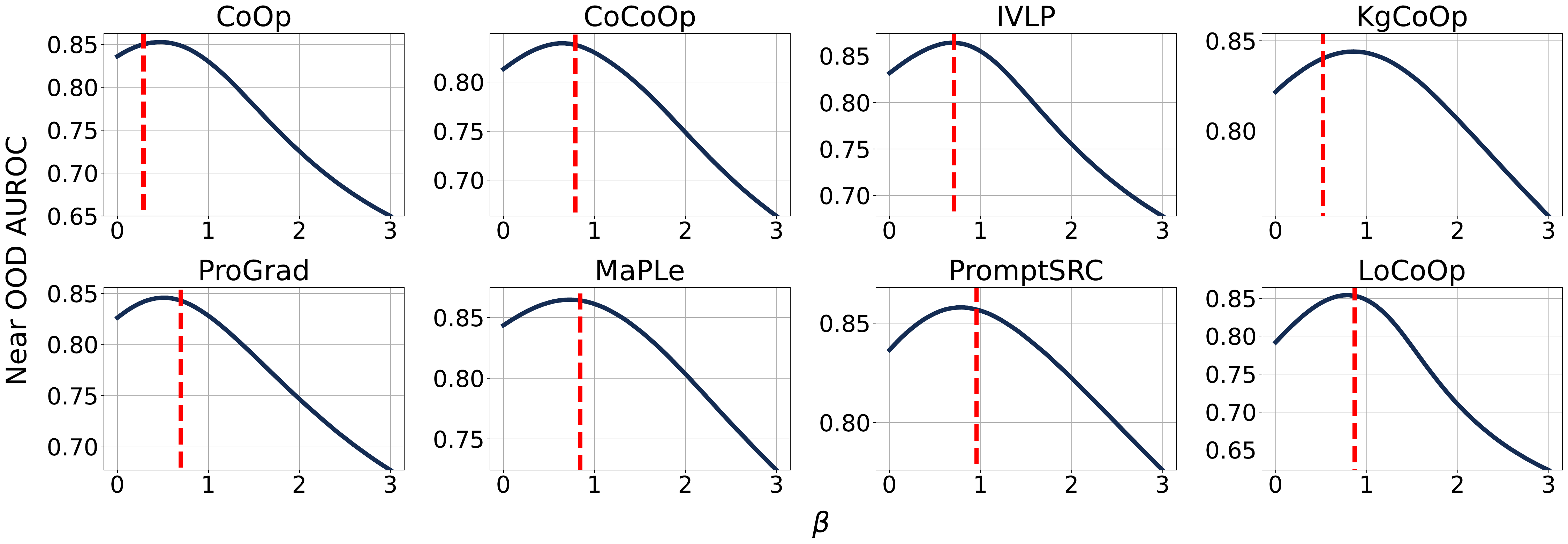}
    \caption{Near OOD detection AUROC using CLS vs. $\beta$ for CoOp~\citep{zhou2022learning}, CoCoOp~\citep{zhou2022conditional}, IVLP~\citep{khattak2023maple}, KgCoOp~\citep{yao2023visual}, ProGrad~\citep{zhu2023prompt}, MaPLe~\citep{khattak2023maple}, PromptSRC~\citep{khattak2023self}, and LoCoOp~\citep{miyai2023locoop} on 16-shots UCF101~\citep{soomro2012ucf101}. The scaling factor is approximated by Eq.(\ref{eq:scale}), shown as red dotted lines.} 
    \label{fig:scale}
\end{figure*}

\subsection{Learning Dataset-Specific Scaling Factor}

Figure~\ref{fig:calibration} shows that selecting a proper value of the scaling factor could increase the performance. 
We argue that the appropriate value of $\beta$ is inherently dataset dependent because it compensates for discrepancies between the surrogate ID and OOD likelihoods that vary across datasets. Each dataset induces different feature distributions, different levels of class concentration, and different forms of overlap between in-distribution structure and the surrogate null distribution. As a result, the raw magnitudes and variances of the surrogate log-likelihoods need not be comparable across datasets, and the degree to which the null score captures distributional uncertainty differs as well. To further demonstrate the impact of $\beta$, we show how the near OOD performance in AUROC of different models varies with different $\beta$ in Figure~\ref{fig:scale}. Notably, a prompt learning model can be highly sensitive to $\beta$, underscoring the critical need for an effective scaling strategy.

This naturally opens a question on how to set $\beta$ properly for each dataset. If we were given the OOD samples, we could simply choose the value of $\beta$ that minimises the near OOD performance. However, such an approach is impractical for real-world applications where near OOD samples are unavailable before deployment. 

To address this, we propose to estimate the margin scale by only using \emph{few-shot ID training samples}. Initially, MaxLogit score and $\langle I,V\rangle$ in Figure~\ref{fig:maxlogit 2d} exhibit positive correlation, leading to significant overlap in the density distributions of ID and near OOD samples. When the near OOD detection AUROC is maximised, as shown in Figure~\ref{fig:CLS-M with beta}, this correlation is minimised, resulting in better separation between ID and near OOD distributions. Thus, we formulate this problem as finding the margin scale that minimises the correlation between MaxLogit and $\langle I,V\rangle$. We propose approximating this correlation using the covariance matrix of a bivariate normal distribution fitted with CLS and $\langle I,V\rangle$ computed from the training samples. A key advantage of this approach is its computational simplicity and the availability of a closed-form solution via maximum likelihood estimation. Specifically, we determine the optimal scaling factor by zeroing out the off-diagonal terms of the covariance matrix of the fitted bivariate normal distribution. This ensures that the MaxLogit score and $\langle I,V\rangle$ become uncorrelated, leading to improved near OOD detection.

\begin{lemma}
\label{lemma:bivariate}
    Given $N$ scalar observations $\{\hat{x}_i\}_{i=1}^N$ and $\{\hat{y}_i\}_{i=1}^N$, we define two variables $x=\hat{x}$ and $y=\hat{y}-\beta \cdot \hat{x}$. The scale parameter $\beta$ that zeros out the covariance of two variables (i.e., the off-diagonals of a covariance matrix) which is approximated by maximum likelihood estimation is:
    \begin{equation}
        \beta =\frac{\sum_{i=1}^N{(\hat{x}_i-\mu_{\hat{x}})(\hat{y}_i-\mu_{\hat{y}})}}{\sum_{i=1}^N{(\hat{x}_i-\mu_{\hat{x}})^2}}
        \label{eq:scale}
    \end{equation}
    where $\mu_{\hat{x}}=\frac{1}{N}\sum_{i=1}^N\hat{x}_i$ and $\mu_{\hat{y}}=\frac{1}{N}\sum_{i=1}^N\hat{y}_i$.
\end{lemma}

By using Lemma~\ref{lemma:bivariate} with $\hat{y}$ being the MaxLogit score and $\hat{x}$ being $\langle I,V\rangle$, $\beta$ can be easily estimated with ID training samples (see Appendix~\ref{appendix:proof} for proof of the lemma). 
Figure~\ref{fig:scale} shows the estimated scaling factor in red dotted lines, demonstrating that our estimation is very close to the empirically optimal value that maximises near OOD detection performance. In addition to good accuracy, our method is a close-form estimation that only takes a small number of ID training samples with negligible computational cost.

\subsection{Connections to Existing Score Functions}
\subsubsection{Energy Score} 
Our proposed framework can also accommodate another widely-used score function: 
Energy score~\citep{liu2020energy} defined as $\tau \log{\sum_{i=1}^K{\exp{(\langle I, P_i\rangle/\tau)}}}$, where $\tau$ is a temperature scaling parameter. It is well known that when $\tau=1$, $\max_{i=1}^K \langle I, P_i\rangle\leq \log{\sum_{i=1}^K{\exp{(\langle I, P_i\rangle)}}} \leq \max_{i=1}^K \langle I, P_i\rangle +\log{K}$. Based on this, we introduce two variants of CLS: CLS-M derived from the MaxLogit score and CLS-E derived from the Energy score:
\begin{align}
    &\text{CLS-M}=\max_{i=1}^K \langle I, P_i\rangle - \beta \langle I, V\rangle, \label{eq:CLS-M} \\
    &\text{CLS-E}=\tau \log{\sum_{i=1}^K{\exp{(\langle I, P_i\rangle/\tau)}}}- \beta \langle I, V\rangle. \label{eq:CLS-E}
\end{align}
where $\tau$ is a temperature scale. The overall algorithm is summarised in Algorithm~\ref{al:mls}.

\subsubsection{Relative Mahalanobis Distance Score} Relative Mahalanobis Distance Score (RMDS)~\citep{ren2021simple} also measures the difference of a class-specific score and a class-agnostic score. However, RMD relies on intermediate feature maps from a traditional classifier with a classification head. As a result, it cannot be applied to a CLIP prompt learning model, which fine-tunes context vectors and performs classification by measuring the cosine similarity between image and text embeddings. In contrast, CLS is specifically designed for CLIP-based prompt learning models, leveraging contrastive similarity rather than feature-space distances. This makes CLS a more natural and effective OOD detection score for vision-language models.

\begin{algorithm}[t]
\DontPrintSemicolon
 \caption{CLS Computation}
 \label{al:mls}
 \KwIn{Few-shot training dataset of $N$ image-label pairs of $\{I_i,y_i\}_{i=1}^N$ where $y_i\in\{1,\cdots,K\}$ with $K$ classes, a test image $I_{\text{test}}$, a fine-tuned prompt learning model with learned context vectors $V$.}
 \For{$I_i, y_i$}
    {Compute and store MaxLogit score or Energy score.\;
    Compute and store Context score $S_{\text{Context}}$.\;}
Estimate margin scale $\beta$ by Eq.(\ref{eq:scale})\;
Compute CLS by Eq.(\ref{eq:CLS-M}) and Eq.(\ref{eq:CLS-E}) \;
\end{algorithm}

\section{Related Work}
\label{sec-rw}
\subsection{Vision-Language Models} Vision-language models have significantly advanced in recent years, bridging the gap between visual and textual data. Early approaches, such as image captioning models~\citep{karpathy2015deep,wang2016image,you2016image}, typically used convolutional neural networks (CNNs) to extract visual features and recurrent neural networks (RNNs) to generate descriptive text. The advent of transformers~\citep{vaswani2017attention} handling long-range dependencies more effectively and contrastive learning~\citep{oord2018representation} revolutionised this field. Notably, ALIGN~\citep{jia2021scaling}, CLIP~\citep{radford2021learning}, and LiT~\citep{zhai2022lit} leveraged a contrastive learning framework that aligns image and text embeddings in a multimodal space, allowing for zero-shot learning capabilities and impressive generalisation to unseen tasks and datasets. In this work, we leverage the powerful vision-language model CLIP and extend its near OOD capability.

\subsection{CLIP-based Prompt Learning} Despite the remarkable zero-shot performance of CLIP, CLIP shows inherently unstable classification accuracy that varies by wording of prompt. To mitigate this issue, CoOp~\citep{zhou2022learning} was proposed to optimise a prompt in word embedding space, leveraging prompt learning from the NLP literature. CoCoOp~\citep{zhou2022conditional} identified that CoOp has limited generalisation and proposed to condition image features to the learnable prompt. Subsequently, many studies have proposed different techniques to improve the generalisation~\citep{yao2023visual,zhu2023prompt,khattak2023maple,khattak2023self}. While its generalisation has been largely improved, its OOD detection has been overlooked. LoCoOp~\citep{miyai2023locoop} proposed a OOD regularisation to improve OOD detection performance. Nevertheless,  no study has addressed near OOD detection of prompt learning models.

\subsection{OOD Detection} 
An early work of \citet{hendrycks2017baseline} utilised the maximum softmax probability (MSP) as a score to identify OOD samples. Another notable work is Out-of-DIstribution detector for Neural networks (ODIN)~\citep{liang2018enhancing} which extends MSP by introducing temperature scaling and input pre-processing to enhance separation of the scores from ID samples and OOD samples. Similar to ODIN, Mahalanobis~\citep{lee2018simple} score also uses input pre-processing in addition to measuring distance in feature space. Delving into a more challenging task of near OOD detection, several studies analysed benchmarks of pre-trained networks in near OOD detection~\citep{yang2023full,yang2022openood,zhang2023openood,fort2021exploring}, and different training methods and score functions were proposed for near OOD detection~\citep{ren2021simple,winkens2020contrastive}. Despite significant advancements in OOD detection for traditional classifier-based neural networks, many existing methods are not directly applicable to CLIP-based prompt learning models, which lack classifier heads. Furthermore, since these models do not update their image encoders during fine-tuning, many distance-based methods that rely on image features become ineffective.

\section{Experiments}
\label{sec:experiments}
\subsection{Experimental Settings}
\subsubsection{Datasets} Following previous works of CLIP-based prompt learning models~\citep{zhou2022conditional,khattak2023maple,yao2023visual,zhu2023prompt,khattak2023maple,khattak2023self,miyai2023locoop}, we use 11 publicly available datasets of ImageNet~\citep{deng2009imagenet}, Caltech101~\citep{li2004learning}, OxfordPets~\citep{parkhi2012cats}, StanfordCars~\citep{kraus2320133d}, Flowers102~\citep{nilsback2008automated}, Food101~\citep{bossard2014food}, FGVCAircraft~\citep{maji2013fine}, SUN397~\citep{xiao2010sun}, DTD~\citep{climpoi2014describing}, EuroSAT~\citep{helber2019eurosat}, and UCF101~\citep{soomro2012ucf101}. A common task of these works involves training models on half of the label classes (e.g., base classes) and evaluating them on the other half classes (e.g., new classes) to measure base-to-new generalisation. We reframe this task as a near OOD detection problem. Specifically, the models trained on base classes are tested with a dataset where half of the samples belong to the base classes (ID) and the other half to new classes (near OOD). The task is to detect whether each test image belongs to the ID dataset or the near OOD dataset. In addition, we include CIFAR10~\citep{krizhevsky2009learning} and CIFAR100~\citep{krizhevsky2009learning}, which are the standard near OOD detection benchmarks~\citep{ren2021simple,fort2021exploring,yang2021generalized,yang2022openood,zhang2023openood}. For CIFAR10 and CIFAR100, we use all classes and evaluate with a test dataset consisting of both CIFAR10 test samples and CIAFR100 test samples, following the literature.

\textbf{Data Availability Statement.} All data used in this work is publicly available.

\subsubsection{Base Models and Baselines}

    As a post-hoc scoring function, our method can be applied with most of vision-language few-shot prompt learning models without retraining. We select 8 popular methods as our base models, including: CoOp~\citep{zhou2022learning}, CoCoOp~\citep{zhou2022conditional}, IVLP~\citep{khattak2023maple}, KgCoOp~\citep{yao2023visual}, ProGrad~\citep{zhu2023prompt}, MaPLe~\citep{khattak2023maple}, PromptSRC~\citep{khattak2023self}, and LoCoOp~\citep{miyai2023locoop}. We follow their training details to train them with 16, 8, 4, 2, and 1-shot settings using 3 random seeds.
    
    In addition to MaxLogit~\citep{hendrycks2022scaling} and Energy Score~\citep{liu2020energy}, we also identify NegPrompt~\citep{Li2024learning} and LoCoOp (GL-MCM) ~\citep{miyai2023locoop} as the closet methods to ours, which are few-shot prompt learning models designed for OOD detection.
    
    We also compare our approach with several OOD detection methods developed for vision-language models which are not prompt learning based methods. These include MCM~\citep{ming2022delving}, NegLabel~\citep{jiang2024negative}, AdaNeg~\citep{zhang2024adaneg}, and CLIPN~\citep{wang2023clipn}.

    For references, we also report the results of widely-used OOD detection methods for a ViT-B/16 classifier fully fine-tuned with full training data instead of CLIP including ASH~\citep{djurisic2023extremely}, GradNorm~\citep{huang2021on}, KNN~\citep{pmlr-v162-sun22d}, Mahalanobis distance Score (MDS)~\citep{lee2018simple}, MSP~\citep{hendrycks2017baseline}, ODIN~\citep{liang2018enhancing}, React~\citep{sun2021react}, RMDS~\citep{ren2021simple}, VIM~\citep{Wang_2022_CVPR}. 

    For all models, we use ViT-B/16~\citep{dosovitskiy2021an} for the visual encoder and Transformer~\citep{vaswani2017attention} for the text encoder. Refer to Appendix~\ref{appendix:implementation} for implementation details.
    
\begin{table}[!t]
\centering
\caption{Near OOD AUROC ($\uparrow$) of prompt learning models averaged over 13 datasets using the MaxLogit score and CLS-M.}

\begin{tabular}{cccc}
\toprule
 & MaxLogit & CLS-M & $\triangle$ \\
\midrule
CoOp & 80.74 & 81.84 & \textcolor{blue}{+1.09} \\
CoCoOp & 81.09 & 82.74 & \textcolor{blue}{+1.65} \\
IVLP & 81.12 & 84.34 & \textcolor{blue}{+3.23} \\
KgCoOp & 80.84 & 83.12 & \textcolor{blue}{+2.28} \\
ProGrad & 79.77 & 82.35 & \textcolor{blue}{+2.58} \\
MaPLe & 81.06 & 83.94 & \textcolor{blue}{+2.88} \\
PromptSRC & 83.85 & 85.77 & \textcolor{blue}{+1.92} \\
LoCoOp & 77.55 & 81.74 & \textcolor{blue}{+4.18} \\
\bottomrule
\end{tabular}

\label{table:avg maxlogit}
\end{table}

\subsection{Experimental Results}
\label{sec:results}
\begin{sidewaystable*}

\centering
\caption{Near OOD AUROC ($\uparrow$) of ViT-B/16, CLIP-based zero-shot, and CLIP prompt learning 16-shot. The best-performing results in the last category are bolded, while the second-best are underlined.}

\begin{adjustbox}{scale=0.7,center}
\begin{tabular}{ccccccccccccccc}
\toprule

 & ImageNet & Caltech101 & OxfordPets & StanfordCars & Flowers102 & Food101 & FGVCAircraft & SUN397 & DTD & EuroSAT & UCF101 & CIFAR10 & CIFAR100 & Avg \\
 \midrule
 &\multicolumn{14}{c}{{Fully Fine-tuned ViT-B/16 Classifier}} \\
 ASH & \color{gray} 92.84 & \color{gray}94.30 & \color{gray}91.51 &\color{gray} 87.36 & \color{gray}92.62 & \color{gray}83.97 & \color{gray}50.46 & \color{gray}73.50 & \color{gray}78.64 & \color{gray}{84.23} & \color{gray}{86.74} &\color{gray} 96.54 & \color{gray}88.19 & \color{gray}84.68 \\
GradNorm & \color{gray}87.63 & \color{gray}93.41 & \color{gray}92.00 & \color{gray}87.50 & \color{gray}93.75 &\color{gray} 83.03 & \color{gray}44.44 & \color{gray}73.73 & \color{gray}77.43 & \color{gray}79.02 & \color{gray}85.62 & \color{gray}88.23 & \color{gray}80.25 & \color{gray}82.00 \\
KNN & \color{gray}92.84 & \color{gray}73.42 & \color{gray}73.93 & \color{gray}40.12 &\color{gray} 52.50 &\color{gray} 79.38 & \color{gray}68.52 & \color{gray}58.97 & \color{gray}42.08 &\color{gray} 76.14 &\color{gray} 73.79 & \color{gray}{96.82} & \color{gray}85.85 & \color{gray}70.34 \\
MDS & \color{gray}{96.22} &\color{gray} 94.21 &\color{gray} 86.14 &\color{gray} 84.72 &\color{gray} 86.68 &\color{gray} 75.59 &\color{gray} 68.92 &\color{gray} 73.17 &\color{gray} 78.55 &\color{gray} 77.76 &\color{gray} 83.05 &\color{gray} 96.70 &\color{gray} {90.16} &\color{gray} 83.99 \\
MSP &\color{gray} 86.28 &\color{gray} 93.77 &\color{gray} 89.16 &\color{gray} 86.99 &\color{gray} 94.23 &\color{gray} 82.12 &\color{gray} 42.36 &\color{gray} 70.80 &\color{gray} 77.05 &\color{gray} 81.97 &\color{gray} 82.92 &\color{gray} 95.03 &\color{gray} 82.03 &\color{gray} 81.90 \\
ODIN &\color{gray} 92.34 &\color{gray} 93.99 &\color{gray} 91.13 &\color{gray} 88.18 &\color{gray} 91.84 &\color{gray} 84.03 &\color{gray} 42.89 &\color{gray} 72.33 &\color{gray} 77.46 &\color{gray} 83.74 &\color{gray} 86.11 &\color{gray} 89.97 &\color{gray} 72.42 &\color{gray} 82.03 \\
React &\color{gray} 92.94 &\color{gray} {94.3}2 &\color{gray} 91.73 &\color{gray} 87.31 &\color{gray} 92.09 &\color{gray} 83.75 &\color{gray} 51.28 &\color{gray} 73.60 &\color{gray} {78.62} &\color{gray} 84.20 &\color{gray} {86.47} &\color{gray} 96.72 &\color{gray} 88.63 &\color{gray} 84.74 \\
RMDS &\color{gray} 93.87 &\color{gray} 92.42 &\color{gray} 93.26 &\color{gray} 86.70 &\color{gray} 93.58 &\color{gray} 81.51 &\color{gray} {77.28} &\color{gray} 72.64 &\color{gray} 76.61 &\color{gray} {85.29} &\color{gray} 83.46 &\color{gray} 95.76 &\color{gray} 87.09 &\color{gray} 86.11 \\
VIM &\color{gray} 95.52 &\color{gray} {94.42} &\color{gray} 87.99 &\color{gray} 84.87 &\color{gray} 88.03 &\color{gray} 80.81 &\color{gray} 61.07 &\color{gray} 73.02 &\color{gray} {79.14} &\color{gray} 82.34 &\color{gray} 85.01 &\color{gray} {97.05} &\color{gray} {90.00} &\color{gray} 84.56 \\ \midrule
 &\multicolumn{14}{c}{{Non-Prompt-Learning Zero-shot CLIP}} \\
MCM &\color{gray} 78.38 &\color{gray} 85.85 &\color{gray} 77.29 &\color{gray} 62.88 &\color{gray} 78.57 &\color{gray} 83.55 &\color{gray} 28.07 &\color{gray} 71.48 &\color{gray} 64.72 &\color{gray} 50.97 &\color{gray} 75.67 &\color{gray} 92.35 &\color{gray} 73.55 &\color{gray} 71.03 \\
NegLabel &\color{gray} 94.72 &\color{gray} 85.00 &\color{gray} 89.74 &\color{gray} 86.39 &\color{gray} 83.99 &\color{gray} 91.49 &\color{gray} 70.19 &\color{gray} 70.87 &\color{gray} 61.26 &\color{gray} 53.02 &\color{gray} 76.27 &\color{gray} 91.29 &\color{gray} 66.28 &\color{gray} 78.50 \\
AdaNeg &\color{gray} 94.87 &\color{gray} 52.93 &\color{gray} 88.26 &\color{gray} 83.71 &\color{gray} 74.12 &\color{gray} 90.85 &\color{gray} 69.59 &\color{gray} 70.33 &\color{gray} 54.24 &\color{gray} 53.05 &\color{gray} 77.13 &\color{gray} 91.34 &\color{gray} 65.77 &\color{gray} 74.32 \\
CLIPN &\color{gray} 94.17 &\color{gray} 87.08 &\color{gray} 82.65 &\color{gray} 90.63 &\color{gray} 82.58 &\color{gray} 85.61 &\color{gray} 44.20 &\color{gray} 80.03 &\color{gray} 71.45 &\color{gray} 57.56 &\color{gray} 79.09 &\color{gray} 95.05 &\color{gray} 85.27 &\color{gray} 79.64 \\ \midrule
 &\multicolumn{14}{c}{{CLIP Prompt Learning 16-shot}} \\
MaxLogit & 94.66 & 84.68 & 91.64 & 93.84 & 94.40 & 91.36 & 64.57 & 79.60 & 71.98 & \textbf{80.02} & 84.17 & 91.17 & 84.10 & 85.09 \\
Energy & 94.64 & 82.29 & 91.23 & 93.86 & 93.36 & 91.05 & \underline{71.61} & 78.24 & 70.23 & 78.87 & 82.61 & 89.58 & 82.25 & 84.60 \\
 LoCoOp (GL-MCM) & 94.02 & \textbf{91.70} & 86.62 & 82.90 & 89.54 & 89.58 & 37.23 & 79.89 & \textbf{76.61} & 69.08 & \underline{85.10} & 92.38 & 76.19 & 80.87 \\
NegPrompt & 92.60 & 87.34 & 90.26 & 89.79 & 93.00 & 91.60 & 46.90 & 76.56 & 67.21 & 69.98 & 83.79 & 90.74 & 80.33 & 81.55 \\
CLS-M (Ours)& \underline{95.60} & \underline{90.55} & \textbf{94.70} & \underline{95.35} & \textbf{95.69} & \textbf{92.63} & 66.79 & \textbf{82.05} & \underline{73.59} & \underline{79.52} & \textbf{86.03} & \textbf{93.22} & \textbf{86.24}  & \underline{87.07} \\
CLS-E (Ours)& \textbf{95.88} & 89.74 & \underline{94.58} & \textbf{95.49} & \underline{94.96} & \underline{92.49} & \textbf{74.28} & \underline{81.28} & 72.32 & 78.58 & 84.96 & \underline{92.53} & \underline{85.20} & \textbf{87.10} \\
\bottomrule
\end{tabular}
\end{adjustbox}
\label{tab: table3}
\end{sidewaystable*}

\begin{table*}

\centering
\caption{Far OOD AUROC ($\uparrow$) of CLIP-based zero-shot and CLIP prompt learning 16-shot. The best-performing results of the prompt learning methods are bold, while the second-best are underlined.}
\small
\begin{tabular}{cccccc}
\toprule

 & iNaturalist & SUN & Places & Textures & Avg \\
 \midrule
 &\multicolumn{5}{c}{{Non-Prompt-Learning Zero-shot CLIP}} \\
MCM & \color{gray} 94.59& \color{gray}  92.25& \color{gray}  90.31& \color{gray}  86.12& \color{gray}  90.82 \\
NegLabel & \color{gray}  99.49& \color{gray}  95.49& \color{gray}  91.64& \color{gray} 90.22 & \color{gray}  94.21 \\
AdaNeg & \color{gray} 99.71& \color{gray}  97.44& \color{gray} 94.55  & \color{gray} 94.93  & \color{gray} 96.66 \\
 &\multicolumn{5}{c}{{CLIP Prompt Learning 16-shot}} \\
MaxLogit & 97.16 & 97.44 & 96.88 & 96.21 & 96.92  \\
Energy & 93.97 & 95.60 & 95.74 & 94.57 & 94.97 \\
LoCoOp (GL-MCM) & \textbf{99.14} & \textbf{99.19} & \textbf{97.31} & \textbf{96.87} & \textbf{98.13} \\
CLS-M (Ours) & \underline{98.68} & \underline{98.75} & \underline{96.82} & \underline{96.64} & \underline{97.72}  \\
CLS-E (Ours) & 97.36 & 98.27 & 96.45 & 95.54 & 96.91 \\
\bottomrule
\end{tabular}
\label{tab: table4}
\end{table*}

\subsubsection{Comparison with MaxLogit and Energy Score}
Recall that when $\beta=0$, our proposed CLS-M and CLS-E reduce to MaxLogit and Energy score respectively, and CLS-M and CLS-E are tailored scoring functions for vision-language prompt learning methods.
In Table~\ref{table:avg maxlogit}, we compare near OOD detection AUROC using the MaxLogit score and our proposed CLS-M score with the difference presented in the last column. We report average AUROC with 16, 8, 4, 2, and 1-shot settings in Figure~\ref{fig:avg maxlogit auroc}.
Positive improvements in AUROC were observed in 100 out of 104 evaluations (13 datasets × 8 models) when using CLS-M, demonstrating its robustness across diverse datasets and model architectures. 

Notably, the largest improvement was observed with LoCoOp, a recent prompt learning model designed for OOD detection.
A similar trend is observed in Figure~\ref{fig:avg energy auroc}, where CLS-E improves AUROC in 101 out of 104 evaluations when used with the Energy score. Additionally, for a detailed breakdown of false positive rates (FPR) at 95\% true positive rate (TPR), refer to Figure~\ref{fig:avg maxlogit fpr95} (MaxLogit) and Figure~\ref{fig:avg energy fpr95} (Energy score).

\subsubsection{Comparison with other methods}
Table~\ref{tab: table3} presents a comprehensive comparison of OOD detection performance across three major categories.
For our methods, we use PromptSRC as the base prompt learning model with a 16-shot setting. The last category (i.e., CLIP Prompt Learning 16-shot results) includes the most directly comparable methods to ours which are either scoring functions applicable to CLIP-based prompt learning or methods specifically designed for it. Our proposed methods, CLS-M and CLS-E, achieve the best overall performance in this category, outperforming the others in this category on most of the datasets. The best performing one within the same category is in bold font, and the second best one is underlined.
When comparing across different categories, our methods consistently rank among the top-performing approaches across most datasets. Overall, our results demonstrate that CLS-M and CLS-E provide a superior post-hoc scoring method for OOD detection in prompt learning models, achieving state-of-the-art performance while maintaining computational efficiency.

\subsubsection{Comparisons in the Far OOD Detection Setting} 
Although our method is not specifically designed for far OOD detection, it can be potentially extended to this setting. Following prior OOD detection studies on CLIP~\citep{ming2022delving,miyai2023locoop,wang2023clipn,jiang2024negative}, we use ImageNet as the ID dataset and iNaturalist~\citep{van2018the}, SUN~\citep{xiao2010sun}, Places~\citep{zhou2018places}, and Texture~\citep{climpoi2014describing} as far OOD datasets. We leveraged the fine-tuned models from our experiments and report the corresponding AUROC results in Table~\ref{tab: table4}. LoCoOp serves as the backbone for the CLIP prompt learning methods. The results of the zero-shot methods are taken from~\citet{zhang2024adaneg}. 

Our method achieves the second-best performance, with only a marginal difference from GL-MCM (0.41 average AUROC). Although GL-MCM performs slightly better in far OOD settings, the performance gap is considerably smaller than that observed in near OOD settings (6.23 average AUROC), where our approach outperforms GL-MCM. Since, in practice, the type of OOD data (near or far) is typically unknown, our method offers greater overall reliability.

To provide a more comprehensive view of our approach’s effectiveness across different prompt learning methods, we report detailed AUROC and FPR95 results in Tables~\ref{table:ood auroc} and~\ref{table:ood fpr} in Appendix~\ref{appendix: far ood}. Across these experiments, our approach consistently outperforms the MaxLogit and Energy baselines, achieving improvements in 126 out of 128 evaluations.

\section{Conclusion}
In this work, we address few-shot near OOD detection of CLIP-based prompt learning models,  a crucial challenge for deploying vision-language models in real-world applications. %
To address this issue, we propose a simple and efficient post-hoc method that can be seamlessly applied to any vision-language prompt learning model. Our approach enhances near OOD performance without modifying the model architecture or training procedure, ensuring that classification accuracy remains unaffected. By refining confidence estimations, our method significantly improves near OOD AUROC and reduces FPR95, leading to more reliable OOD detection.
Through extensive experiments on 8 state-of-the-art prompt learning models and 13 real-world datasets, we demonstrate that our method consistently achieve competitive near OOD detection performance while maintaining computational efficiency, in the comparisons with multiple kinds of baselines. 

While our method is broadly applicable across various prompt learning models, the degree of improvement can vary depending on the underlying model characteristics. Some models may not see as substantial a benefit, particularly if their inherent logit distributions already exhibit strong separability between ID and OOD samples. 

\section*{Declarations}
We acknowledge Lan Du for the useful discussions.
\subsubsection*{Funding Information}
Not applicable.

\backmatter

\begin{appendices}

\section{Appendix}
\begin{table*}[!h]
\centering
\caption{Training details of the prompt learning models.}
\begin{tabular}{cccc}
\toprule
& \# Epochs & Batch Size & Context Vectors Initialisation  \\
\midrule
CoOp &  \makecell{50 (ImageNet)\\ 200 (Others)} & 32 & \multirow{7}{*}{``a photo of a''}  \\
CoCoOp &  10 & 1 &   \\
IVLP & 5 & 4 &  \\
KgCoOp & 100 & 128 &  \\
ProGrad & 200 & 32 &  \\
MaPLe & 5 & 4 &  \\
PromptSRC & 20 & 4 &  \\
LoCoOp & 50 & 32 & 16 vectors drawn from $\mathcal{N}(0,0.02)$\\
\bottomrule
\end{tabular}
\label{table:implementation}
\end{table*}
\subsection{Proof of Lemma}
\label{appendix:proof}
We provide the proof of Lemma~\ref{lemma:bivariate}. For completeness of proof, we duplicate the lemma here.
\begin{lemma}
\label{appendix lemma:bivariate}
    Given $N$ scalar observations $\{\hat{x}_i\}_{i=1}^N$ and $\{\hat{y}_i\}_{i=1}^N$, we define two variables $x=\hat{x}$ and $y=\hat{y}-\beta \cdot \hat{x}$. The scale parameter $\beta$ that zeros out the covariance of two variables (i.e., the off-diagonals of a covariance matrix) which is approximated by maximum likelihood estimation is:
    \begin{equation}
        \beta =\frac{\sum_{i=1}^N{(\hat{x}_i-\mu_{\hat{x}})(\hat{y}_i-\mu_{\hat{y}})}}{\sum_{i=1}^N{(\hat{x}_i-\mu_{\hat{x}})^2}}
    \end{equation}
    where $\mu_{\hat{x}}=\frac{1}{N}\sum_{i=1}^N\hat{x}_i$ and $\mu_{\hat{y}}=\frac{1}{N}\sum_{i=1}^N\hat{y}_i$.
\end{lemma}
\begin{proof}
    It is well known that maximum likelihood estimation (MLE) of bivariate normal distribution for $N$ observations of variables $x$ and $y$ results in~\citep{bishop2013pattern}:
    \begin{align}
    &\mu_{x}=\frac{1}{N}\sum_{i}{x_i},\quad \mu_{y}=\frac{1}{N}\sum_{i}{y_i} \\
    &\Sigma=\begin{bmatrix}
        \sigma_{xx} & \sigma_{xy} \\
        \sigma_{xy} & \sigma_{yy}
    \end{bmatrix} \\
    &=\frac{1}{N} \sum_i {\left(\begin{bmatrix}
        x_i \\
        y_i
    \end{bmatrix}-\begin{bmatrix}
        \mu_{x} \\
        \mu_{y}
    \end{bmatrix}\right)\left(\begin{bmatrix}
        x_i \\
        y_i
    \end{bmatrix}-\begin{bmatrix}
        \mu_{x} \\
        \mu_{y}
    \end{bmatrix}\right)^T} \\
    &\sigma_{xy} = \frac{1}{N} \sum_i{(x_i-\mu_x)(y_i-\mu_y)}
\end{align}
where $\mu_{x}$ and $\mu_{y}$ are the means of $x$ and $y$, and $\Sigma$ is the covariance matrix. We let $x=\hat{x}$ and $y=\hat{y}-\beta \cdot \hat{x}$ and find $\beta$ that makes $\sigma_{xy}=0$. By rewriting $\sigma_{xy}$ in terms of $\hat{x}$ and $\hat{y}$, we obtain $\beta$ as:
\begin{align}
    \sigma_{xy}&=\frac{1}{N}\sum_i{(\hat{x}_i-\mu_{\hat{x}})(\hat{y}_i-\beta\cdot \hat{x}_i -\mu_{\hat{y}}+\beta\cdot \mu_{\hat{x}})}=0 \\
    \beta&=\frac{\sum_i{(\hat{x}_i-\mu_{\hat{x}})(\hat{y}_i-\mu_{\hat{y}})}}{\sum_i{(\hat{x}_i-\mu_{\hat{x}})^2}}
\end{align}
The resulting $\beta$ is the ratio of covariance of $\hat{x}$ and $\hat{y}$ to variance of $\hat{x}$.
\end{proof}

\subsection{Implementation Details}
\label{appendix:implementation}
We follow the officially released training guidelines for each prompt learning model using the same configuration files. The only additional line of code required is \texttt{beta=(((y-y.mean())*(x-x.mean())).sum())\\/(((x-x.mean())**2).sum())} to estimate the margin scale in Eq.(\ref{eq:scale}). Table~\ref{table:implementation} shows common hyperparameters which are the number of epochs, batch size, and context vectors initialisation. Refer to their officially released codes for other model-specific hyperparameters. All models were trained on a single NVIDIA GeForce RTX 3090 GPU with PyTorch framework. The temperature scaling is 0.01 for the Energy score and 1 for the MCM score.

For all other OOD methods, we followed the implementations from \url{https://github.com/YBZh/OpenOOD-VLM}, \url{https://github.com/mala-lab/NegPrompt}, and \url{https://github.com/xmed-lab/CLIPN}.

\subsection{Additional Experimental Results}
\label{appendix:additional}

We provide additional experimental results other than the results in the main section. 

\subsubsection{MaxLogit Score} We provide difference in average AUROC and FPR95 across 1-, 2-, 4-, 8-, and 16-shot settings between CLS-M and
the MaxLogit score in Figure~\ref{fig:avg maxlogit auroc} and Figure~\ref{fig:avg maxlogit fpr95}.

\begin{figure*}[!t]
    \centering
    \includegraphics[width=0.95\textwidth]{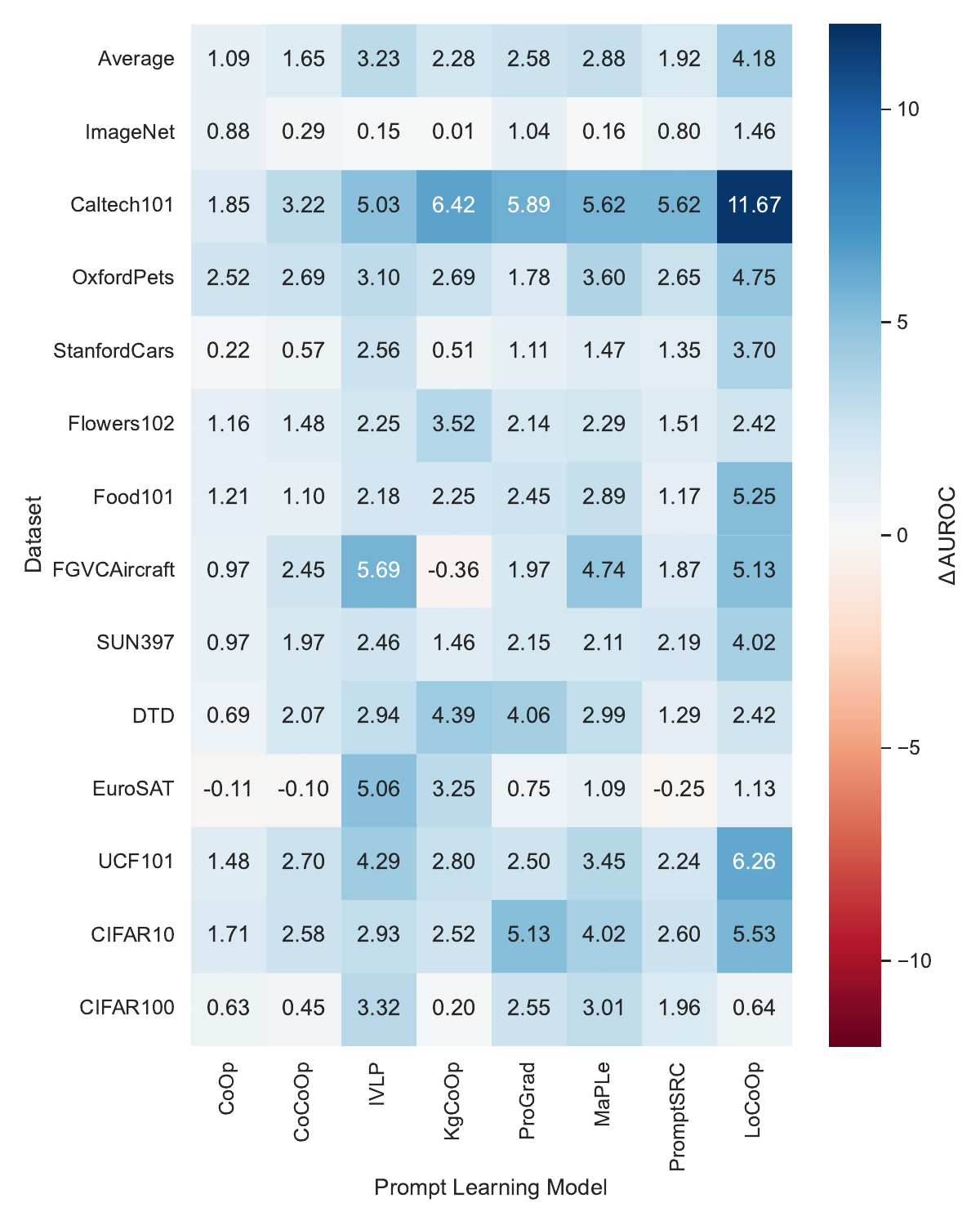}
    \caption{Average difference in AUROC across 1-, 2-, 4-, 8-, and 16-shot settings between CLS-M and the MaxLogit score ($\Delta \text{AUROC}=\text{CLS-M}-\text{MaxLogit}$) evaluated on 13 datasets and 8 prompt learning models. On average across datasets, CLS-M achieves consistent performance gains over MaxLogit across all models.}
    \label{fig:avg maxlogit auroc}
\end{figure*}

\begin{figure*}[!t]
    \centering
    \includegraphics[width=0.95\textwidth]{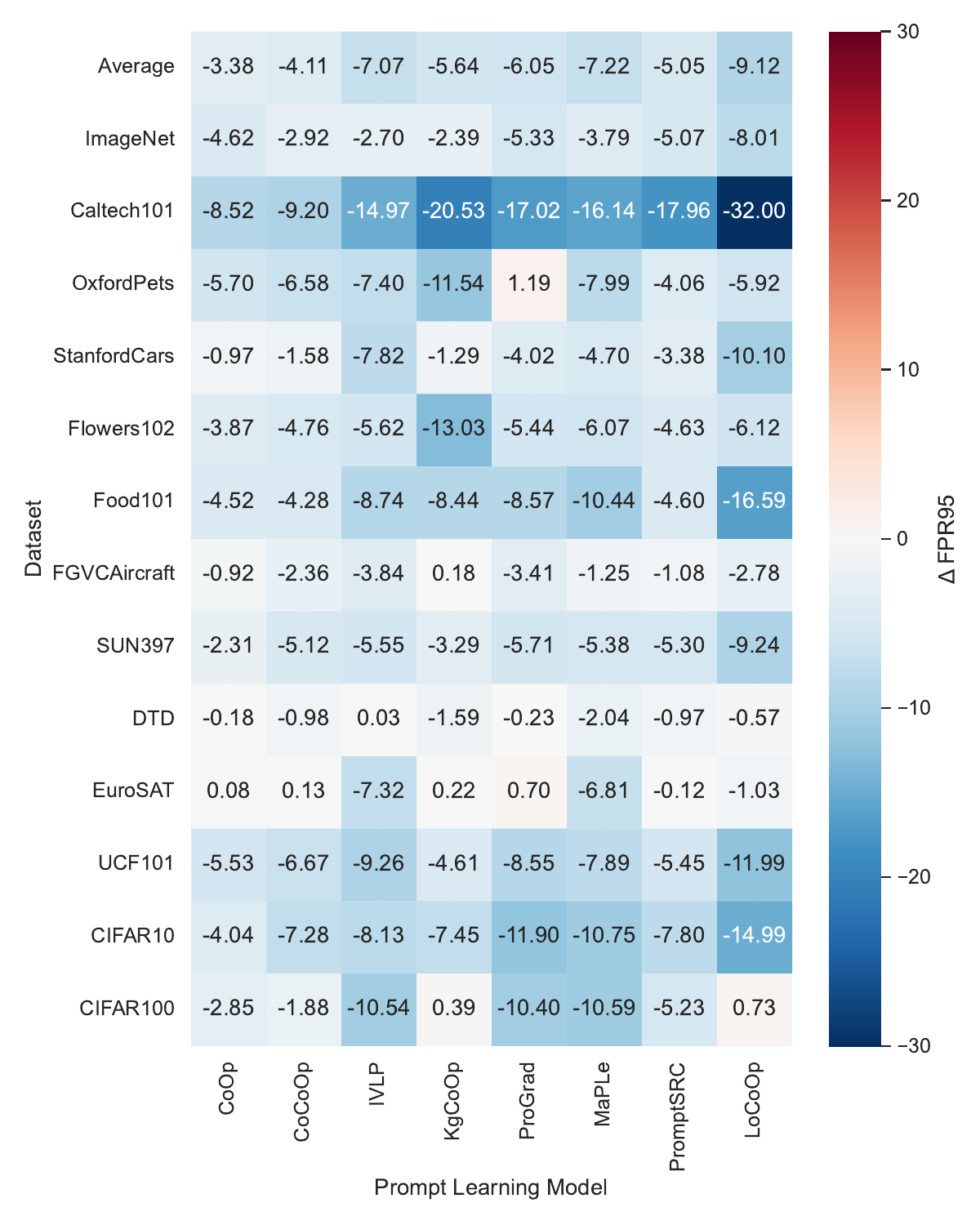}
    \caption{Average difference in FPR95 across 1-, 2-, 4-, 8-, and 16-shot settings between CLS-M and the MaxLogit score ($\Delta \text{FPR95}=\text{CLS-M}-\text{MaxLogit}$) evaluated on 13 datasets and 8 prompt learning models. On average across datasets, CLS-M achieves consistent performance gains over MaxLogit across all models.}
    \label{fig:avg maxlogit fpr95}
\end{figure*}

\subsubsection{Energy Score}
We provide the same results of Figure~\ref{fig:avg maxlogit auroc} and Figure~\ref{fig:avg maxlogit fpr95} using Energy score in Figure~\ref{fig:avg energy auroc} and Figure~\ref{fig:avg energy fpr95}.

\begin{figure*}[!t]
    \centering
    \includegraphics[width=0.95\textwidth]{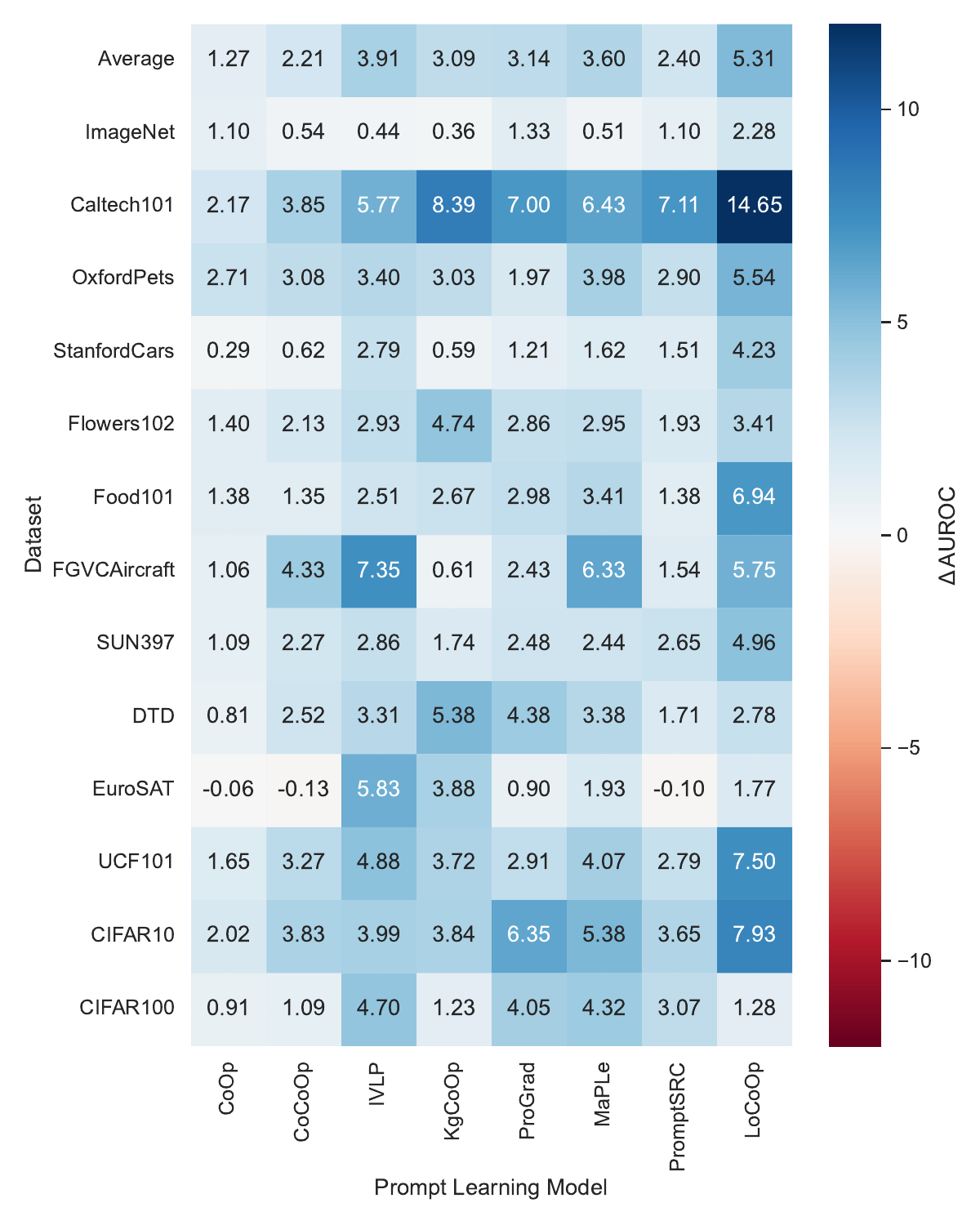}
    \caption{Average difference in AUROC across 1-, 2-, 4-, 8-, and 16-shot settings between CLS-E and the Energy score ($\Delta \text{AUROC}=\text{CLS-E}-\text{Energy}$) evaluated on 13 datasets and 8 prompt learning models. On average across datasets, CLS-E achieves consistent performance gains over Energy across all models.}
    \label{fig:avg energy auroc}
\end{figure*}

\begin{figure*}[!t]
    \centering
    \includegraphics[width=0.95\textwidth]{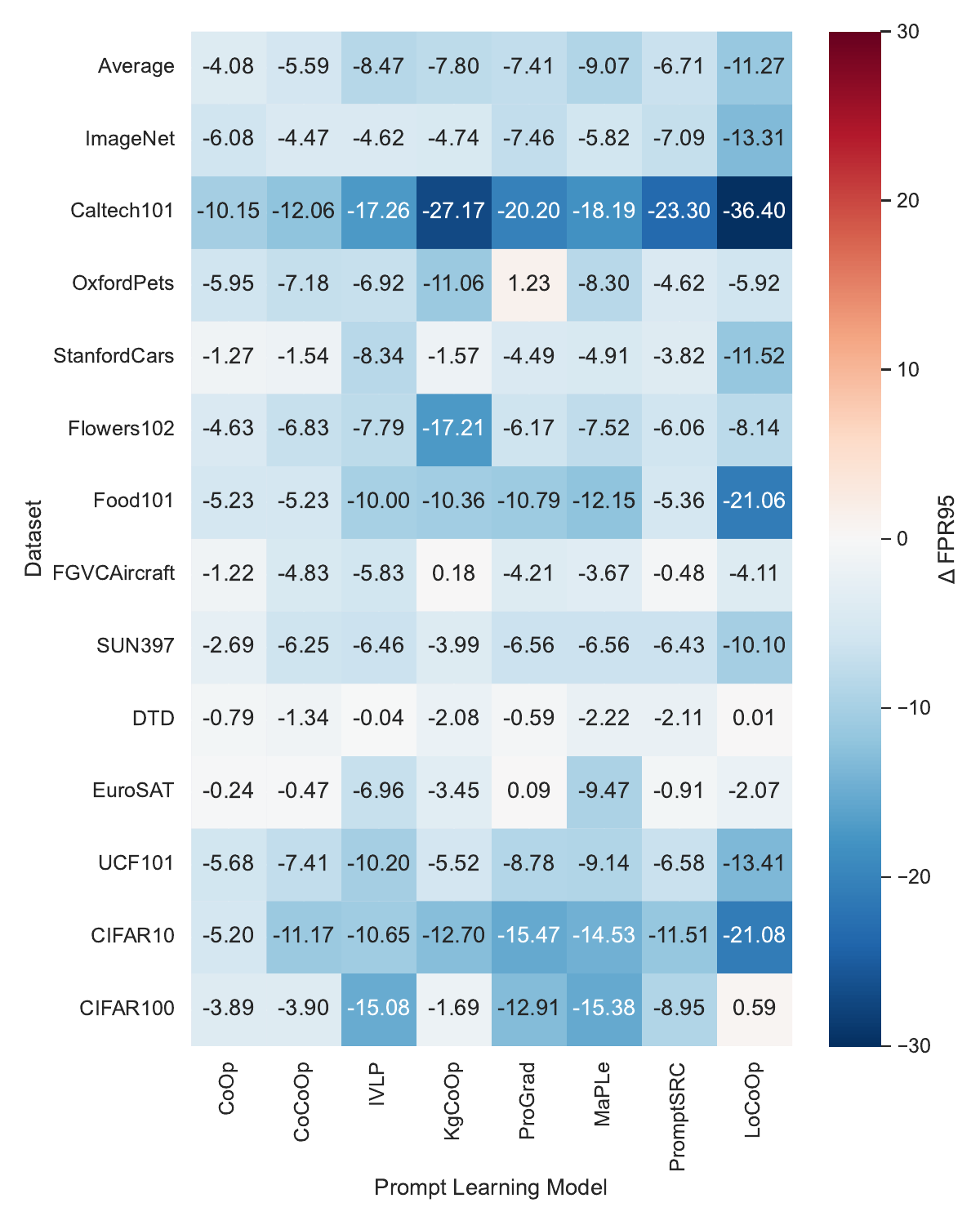}
    \caption{Average difference in FPR95 across 1-, 2-, 4-, 8-, and 16-shot settings between CLS-E and the Energy score ($\Delta \text{FPR95}=\text{CLS-E}-\text{Energy}$) evaluated on 13 datasets and 8 prompt learning models. On average across datasets, CLS-E achieves consistent performance gains over Energy across all models.}
    \label{fig:avg energy fpr95}
\end{figure*}

\subsubsection{Far OOD Detection} 
\label{appendix: far ood}
We provide additional far OOD detection results of AUROC and FPR in Table~\ref{table:ood auroc} and Table~\ref{table:ood fpr} respectively. 

\begin{table*}[!h]
\centering
\caption{Far OOD detection AUROC ($\uparrow$) averaged over 16, 8, 4, 2, and 1 few-shot settings with 3 random seeds.}
\begin{subtable}[t]{0.8\textwidth}
\caption{Places.}
\centering
\centering\resizebox{0.7\textwidth}{!}{

\begin{tabular}{ccccccc}
\toprule
 & \multicolumn{3}{c}{MaxLogit} & \multicolumn{3}{c}{Energy} \\  \cmidrule(l){2-4} \cmidrule(l){5-7}   & Original & CLS-M & $\triangle$ & Original & CLS-E & $\triangle$ \\
\midrule
CoOp & 96.22 & {96.78} & \textcolor{blue}{+0.56} & 95.80 & 96.65 & \textcolor{blue}{+0.85} \\
CoCoOp & {96.90} & 96.79 & \textcolor{red}{-0.11} & 96.50 & 96.51 & \textcolor{blue}{+0.01}  \\
IVLP & 95.79 & {96.01} & \textcolor{blue}{+0.22} & 95.18 & 95.77 & \textcolor{blue}{+0.60}  \\
KgCoOp & 96.67 & {96.91} & \textcolor{blue}{+0.24} & 96.12 & 96.74 & \textcolor{blue}{+0.62} \\
ProGrad & 95.21 & {96.05} & \textcolor{blue}{+0.84} & 94.52 & 95.76 & \textcolor{blue}{+1.25}  \\
MaPLe & {96.61} & 96.52 & \textcolor{red}{-0.09} & 96.00 & 96.11 & \textcolor{blue}{+0.11} \\
PromptSRC & 96.54 & {96.59} & \textcolor{blue}{+0.05} & 96.00 & 96.20 & \textcolor{blue}{+0.21}  \\
LoCoOp & 94.67 & {97.10} & \textcolor{blue}{+2.43} & 91.59 & 96.45 & \textcolor{blue}{+4.86}  \\
\bottomrule
\end{tabular}
}
\end{subtable}
\hfill\begin{subtable}[t]{0.8\textwidth}
\caption{SUN.}

\centering\resizebox{0.7\textwidth}{!}{
\centering
\begin{tabular}{ccccccc}
\toprule
 & \multicolumn{3}{c}{MaxLogit} & \multicolumn{3}{c}{Energy} \\  \cmidrule(l){2-4} \cmidrule(l){5-7}   & Original & CLS-M & $\triangle$ & Original & CLS-E & $\triangle$ \\
\midrule
CoOp & 96.49 & 97.88 & \textcolor{blue}{+1.39} & 95.69 & 97.62 & \textcolor{blue}{+1.93} \\
CoCoOp & 98.17 & {98.47} & \textcolor{blue}{+0.30} & 97.64 & 98.21 & \textcolor{blue}{+0.57} \\
IVLP & 97.61 & 98.20 & \textcolor{blue}{+0.58} & 96.98 & 97.97 & \textcolor{blue}{+0.99}  \\
KgCoOp & 97.50 & {98.47} & \textcolor{blue}{+0.96} & 96.61 & 98.29 & \textcolor{blue}{+1.68} \\
ProGrad & 96.51 & {97.73} & \textcolor{blue}{+1.22} & 95.60 & 97.35 & \textcolor{blue}{+1.75}  \\
MaPLe & 97.91 & {98.54} & \textcolor{blue}{+0.63} & 97.13 & 98.21 & \textcolor{blue}{+1.07} \\
PromptSRC & 97.63 & 98.11 & \textcolor{blue}{+0.48} & 96.88 & 97.64 & \textcolor{blue}{+0.76}  \\
LoCoOp & 95.38 & {98.94} & \textcolor{blue}{+3.56} & 90.79 & 98.20 & \textcolor{blue}{+7.40} \\
\bottomrule
\end{tabular}
}
\end{subtable}
\hfill\begin{subtable}[t]{0.8\textwidth}
\caption{Texture.}
\centering\resizebox{0.7\textwidth}{!}{
\centering
\begin{tabular}{ccccccc}
\toprule
 & \multicolumn{3}{c}{MaxLogit} & \multicolumn{3}{c}{Energy} \\  \cmidrule(l){2-4} \cmidrule(l){5-7}   & Original & CLS-M & $\triangle$ & Original & CLS-E & $\triangle$ \\
\midrule
CoOp & 93.38 & 93.74 & \textcolor{blue}{+0.36} & 92.51 & 93.07 & \textcolor{blue}{+0.56}\\
CoCoOp & 94.79 & 95.29 & \textcolor{blue}{+0.50} & 93.72 & 94.51 & \textcolor{blue}{+0.79}  \\
IVLP & 93.32 & 95.49 & \textcolor{blue}{+2.17} & 91.78 & 94.90 & \textcolor{blue}{+3.12}  \\
KgCoOp & 93.59 & 96.30 & \textcolor{blue}{+2.71} & 92.03 & 95.90 & \textcolor{blue}{+3.86}  \\
ProGrad & 91.36 & 94.22 & \textcolor{blue}{+2.86} & 89.72 & 93.37 & \textcolor{blue}{+3.65} \\
MaPLe & 93.36 & 93.89 & \textcolor{blue}{+0.52} & 91.73 & 92.64 & \textcolor{blue}{+0.91} \\
PromptSRC & 92.69 & 94.52 & \textcolor{blue}{+1.83} & 90.94 & 93.52 & \textcolor{blue}{+2.58}  \\
LoCoOp & 93.18 & 96.16 & \textcolor{blue}{+2.98} & 89.80 & 94.82 & \textcolor{blue}{+5.02}  \\
\bottomrule
\end{tabular}
}
\end{subtable}
\begin{subtable}[t]{0.8\textwidth}
\centering
\caption{iNaturalist.}
\centering\resizebox{0.7\textwidth}{!}{\centering
\centering
\begin{tabular}{ccccccc}
\toprule
 & \multicolumn{3}{c}{MaxLogit} & \multicolumn{3}{c}{Energy} \\  \cmidrule(l){2-4} \cmidrule(l){5-7}   & Original & CLS-M & $\triangle$ & Original & CLS-E & $\triangle$ \\
\midrule
CoOp & 94.58 & 97.23 & \textcolor{blue}{+2.65} & 92.66 & 96.28 & \textcolor{blue}{+3.62} \\
CoCoOp & 97.02 & {97.65} & \textcolor{blue}{+0.63} & 95.61 & 96.58 & \textcolor{blue}{+0.97}  \\
IVLP & 95.98 & 97.88 & \textcolor{blue}{+1.91} & 93.74 & 96.89 & \textcolor{blue}{+3.15} \\
KgCoOp & 97.05 & {98.49} & \textcolor{blue}{+1.44} & 95.61 & 98.00 & \textcolor{blue}{+2.39}\\
ProGrad & 92.90 & 96.04 & \textcolor{blue}{+3.14} & 90.46 & 94.73 & \textcolor{blue}{+4.27} \\
MaPLe & 95.80 & 97.60 & \textcolor{blue}{+1.80} & 93.43 & 96.25 & \textcolor{blue}{+2.82}  \\
PromptSRC & 96.37 & 96.76 & \textcolor{blue}{+0.38} & 94.41 & 95.01 & \textcolor{blue}{+0.60}  \\
LoCoOp & 93.41 & {98.46} & \textcolor{blue}{+5.05} & 87.11 & 97.08 & \textcolor{blue}{+9.97}\\
\bottomrule
\end{tabular}
}
\end{subtable}

\label{table:ood auroc}
\end{table*}

\begin{table*}[!h]
\centering
\caption{Far OOD detection FPR95 ($\downarrow$) averaged over 16, 8, 4, 2, and 1 few-shot settings with 3 random seeds.}
\begin{subtable}[t]{0.8\textwidth}
\caption{Places.}
\centering\resizebox{0.7\textwidth}{!}{
\centering
\begin{tabular}{ccccccc}
\toprule
 & \multicolumn{3}{c}{MaxLogit} & \multicolumn{3}{c}{Energy} \\  \cmidrule(l){2-4} \cmidrule(l){5-7}   & Original & CLS-M & $\triangle$ & Original & CLS-E & $\triangle$ \\
\midrule
CoOp & 17.58 & {13.56} & \textcolor{blue}{-4.01} & 20.90 & 14.42 & \textcolor{blue}{-6.48}  \\
CoCoOp & 13.01 & {12.92} & \textcolor{blue}{-0.09} & 15.41 & 14.09 & \textcolor{blue}{-1.33} \\
IVLP & 17.82 & {15.32} & \textcolor{blue}{-2.50} & 22.14 & 16.67 & \textcolor{blue}{-5.47} \\
KgCoOp & 14.43 & {12.46} & \textcolor{blue}{-1.97} & 18.32 & 13.33 & \textcolor{blue}{-4.99}  \\
ProGrad & 22.47 & {16.28} & \textcolor{blue}{-6.19} & 28.38 & 17.85 & \textcolor{blue}{-10.52}  \\
MaPLe & 15.37 & {14.01} & \textcolor{blue}{-1.37} & 19.35 & 16.74 & \textcolor{blue}{-2.62}  \\
PromptSRC & 14.88 & {13.77} & \textcolor{blue}{-1.11} & 18.70 & 16.04 & \textcolor{blue}{-2.66} \\
LoCoOp & 27.58 & {11.05} & \textcolor{blue}{-16.52} & 53.58 & 14.25 & \textcolor{blue}{-39.33} \\
\bottomrule
\end{tabular}
}
\end{subtable}
\hfill\begin{subtable}[t]{0.8\textwidth}
\caption{SUN.}
\centering\resizebox{0.7\textwidth}{!}{
\centering
\begin{tabular}{ccccccc}
\toprule
 & \multicolumn{3}{c}{MaxLogit} & \multicolumn{3}{c}{Energy} \\  \cmidrule(l){2-4} \cmidrule(l){5-7}   & Original & CLS-M & $\triangle$ & Original & CLS-E & $\triangle$ \\
\midrule
CoOp & 19.44 & {10.06} & \textcolor{blue}{-9.38} & 26.66 & 11.60 & \textcolor{blue}{-15.06}\\
CoCoOp & 8.37 & {6.39} & \textcolor{blue}{-1.98} & 11.64 & 7.18 & \textcolor{blue}{-4.46}  \\
IVLP & 10.77 & 7.92 & \textcolor{blue}{-2.86} & 15.85 & 9.35 & \textcolor{blue}{-6.50}  \\
KgCoOp & 12.98 & {6.81} & \textcolor{blue}{-6.17} & 20.58 & 7.32 & \textcolor{blue}{-13.26} \\
ProGrad & 19.14 & {10.37} & \textcolor{blue}{-8.77} & 27.60 & 12.33 & \textcolor{blue}{-15.28} \\
MaPLe & 9.67 & {5.91} & \textcolor{blue}{-3.76} & 15.09 & 7.69 & \textcolor{blue}{-7.40} \\
PromptSRC & 11.48 & 8.37 & \textcolor{blue}{-3.11} & 17.35 & 11.28 & \textcolor{blue}{-6.08}\\
LoCoOp & 29.82 & {3.76} & \textcolor{blue}{-26.06} & 64.29 & 7.80 & \textcolor{blue}{-56.49} \\
\bottomrule
\end{tabular}
}
\end{subtable}
\hfill\begin{subtable}[t]{0.8\textwidth}
\caption{Texture.}
\centering\resizebox{0.7\textwidth}{!}{
\centering
\begin{tabular}{ccccccc}
\toprule
 & \multicolumn{3}{c}{MaxLogit} & \multicolumn{3}{c}{Energy} \\  \cmidrule(l){2-4} \cmidrule(l){5-7}   & Original & CLS-M & $\triangle$ & Original & CLS-E & $\triangle$ \\
\midrule
CoOp & 35.02 & 31.39 & \textcolor{blue}{-3.63} & 43.04 & 37.32 & \textcolor{blue}{-5.72}  \\
CoCoOp & 26.11 & 23.14 & \textcolor{blue}{-2.96} & 36.68 & 29.31 & \textcolor{blue}{-7.37} \\
IVLP & 33.43 & 20.25 & \textcolor{blue}{-13.18} & 43.80 & 24.40 & \textcolor{blue}{-19.40}  \\
KgCoOp & 33.06 & 18.30 & \textcolor{blue}{-14.76} & 46.06 & 22.26 & \textcolor{blue}{-23.79}  \\
ProGrad & 42.31 & 28.67 & \textcolor{blue}{-13.64} & 54.27 & 36.97 & \textcolor{blue}{-17.30}  \\
MaPLe & 35.20 & 29.43 & \textcolor{blue}{-5.77} & 47.83 & 40.47 & \textcolor{blue}{-7.36} \\
PromptSRC & 37.51 & 26.90 & \textcolor{blue}{-10.61} & 52.59 & 36.40 & \textcolor{blue}{-16.19}  \\
LoCoOp & 39.95 & 19.69 & \textcolor{blue}{-20.25} & 68.07 & 32.19 & \textcolor{blue}{-35.87}  \\
\bottomrule
\end{tabular}
}
\end{subtable}
\begin{subtable}[t]{0.8\textwidth}\centering
\caption{iNaturalist.}
\centering\resizebox{0.7\textwidth}{!}{
\begin{tabular}{ccccccc}
\toprule
 & \multicolumn{3}{c}{MaxLogit} & \multicolumn{3}{c}{Energy} \\  \cmidrule(l){2-4} \cmidrule(l){5-7}   & Original & CLS-M & $\triangle$ & Original & CLS-E & $\triangle$ \\
\midrule
CoOp & 32.70 & 13.64 & \textcolor{blue}{-19.07} & 48.16 & 20.24 & \textcolor{blue}{-27.93}  \\
CoCoOp & 14.03 & {9.68} & \textcolor{blue}{-4.35} & 24.75 & 16.71 & \textcolor{blue}{-8.04} \\
IVLP & 22.51 & 8.85 & \textcolor{blue}{-13.66} & 39.50 & 15.17 & \textcolor{blue}{-24.33}  \\
KgCoOp & 14.76 & {5.76} & \textcolor{blue}{-9.00} & 25.50 & 8.60 & \textcolor{blue}{-16.90}  \\
ProGrad & 43.71 & 21.29 & \textcolor{blue}{-22.42} & 58.34 & 32.44 & \textcolor{blue}{-25.90}\\
MaPLe & 23.97 & 10.38 & \textcolor{blue}{-13.58} & 42.23 & 19.62 & \textcolor{blue}{-22.62}  \\
PromptSRC & 19.38 & 16.03 & \textcolor{blue}{-3.35} & 33.65 & 29.98 & \textcolor{blue}{-3.67}  \\
LoCoOp & 51.09 & {6.17} & \textcolor{blue}{-44.92} & 85.97 & 18.46 & \textcolor{blue}{-67.51}\\
\bottomrule
\end{tabular}
}
\end{subtable}
\label{table:ood fpr}

\end{table*}

\subsubsection{ImageNet Protocol Results}
In addition to 13 datasets used in the main experiments, we also provide experimental results on ImageNet Protocol~\citep{Palechor2023large} in Table~\ref{table:imagenet protocol}. We follow the four-split setting used by \citet{Li2024learning}.

\begin{table*}[!h]
\centering
\caption{OOD AUROC ($\uparrow$) of 8 prompt learning models averaged over 4 ImageNet protocol datasets using the MaxLogit score, the Energy score, CLS-M, and CLS-E.}
\begin{tabular}{ccccc}
\toprule
& MaxLogit  & CLS-M & Energy & CLS-E  \\ 
\midrule {CoOp}  &  96.46  &  \textbf{96.74}  &  96.43  &  96.72   \\ {CoCoOp} & 97.42 &  \textbf{97.69} &  97.34  &  97.65   \\ {IVLP} & 97.19  &  \textbf{97.60}  &  97.03  &  97.50  \\ {KgCoOp} & 97.34  &  \textbf{97.57}  &  97.24  &  97.50  \\ {ProGrad} & 96.84  & \textbf{97.20}  & 96.73  & 97.12  \\ {MaPLe} & \textbf{97.48} & \textbf{97.48} & 97.38 & 97.23 \\ {PromptSRC} & 97.57 & \textbf{97.73} & 97.49 & 97.66  \\ {LoCoOp} & 96.59 & \textbf{97.19} & 96.15 & 96.94  \\
\bottomrule
\end{tabular}
\label{table:imagenet protocol}
\end{table*}
\citep{pmlr-v202-li23q}\citep{Patashnik_2021_ICCV}\citep{Xu_2022_CVPR}
\end{appendices}

\clearpage
\newpage
\bibliography{sn-bibliography}

\end{document}